\documentclass[11pt]{article}
\usepackage[margin=1in]{geometry}

\usepackage{etoolbox}

\newtoggle{arxiv}

\toggletrue{arxiv}

\usepackage{macros}
\bluehyperref

\usepackage[style=alphabetic,backend=bibtex,maxnames=12,maxbibnames=10,maxcitenames=10,maxalphanames=10,giveninits=true,doi=false,url=true]{biblatex}
\newcommand*{\citet}[1]{\AtNextCite{\AtEachCitekey{\defcounter{maxnames}{2}}} \textcite{#1}}

\newcommand*{\citep}[1]{\cite{#1}}

\bibliography{bib}
\let\citealp\citep

\usepackage[noend]{algorithmic}
\usepackage{algorithm}

\usepackage{comment}

\usepackage{tikz}
\usetikzlibrary{positioning}
\usetikzlibrary{calc}

\newcommand{\convexpath}[2]{
[   
    create hullnodes/.code={
        \global\edef\namelist{#1}
        \foreach [count=\counter] \nodename in \namelist {
            \global\edef\numberofnodes{\counter}
            \node at (\nodename) [draw=none,name=hullnode\counter] {};
        }
        \node at (hullnode\numberofnodes) [name=hullnode0,draw=none] {};
        \pgfmathtruncatemacro\lastnumber{\numberofnodes+1}
        \node at (hullnode1) [name=hullnode\lastnumber,draw=none] {};
    },
    create hullnodes
]
($(hullnode1)!#2!-90:(hullnode0)$)
\foreach [
    evaluate=\currentnode as \previousnode using \currentnode-1,
    evaluate=\currentnode as \nextnode using \currentnode+1
    ] \currentnode in {1,...,\numberofnodes} {
-- ($(hullnode\currentnode)!#2!-90:(hullnode\previousnode)$)
  let \p1 = ($(hullnode\currentnode)!#2!-90:(hullnode\previousnode) - (hullnode\currentnode)$),
    \n1 = {atan2(\y1,\x1)},
    \p2 = ($(hullnode\currentnode)!#2!90:(hullnode\nextnode) - (hullnode\currentnode)$),
    \n2 = {atan2(\y2,\x2)},
    \n{delta} = {-Mod(\n1-\n2,360)}
  in 
    {arc [start angle=\n1, delta angle=\n{delta}, radius=#2]}
}
-- cycle
}    
    
\usepackage{mathtools}

\newcommand{\diffp}{\varepsilon}  
\newcommand{\ed}{\ensuremath{(\diffp,\delta)}}

\newcommand{\eps}{\diffp}
\usepackage{hyperref}
\usepackage[capitalize]{cleveref}
\usepackage{crossreftools}
\newcommand{\A}{\mathcal{A}}
\newcommand{\cO}{\mathcal{O}}

\renewcommand{\ss}{\eta} 

\renewcommand{\dkl}{D_{\rm kl}}
\newcommand{\db}{D_{\rm h}} 

\newcommand{\Ds}{\mathcal{S}} 
\newcommand{\ds}{z} 
\newcommand{\domain}{\mc{Z}} 
\newcommand{\range}{\mc{T}} 

\newcommand{\simplex}{\Delta}
\newcommand{\xdomain}{\mc{X}} 
\newcommand{\diam}{\mathsf{diam}} 
\newcommand{\opt}{^\star} 
\newcommand{\noise}{\zeta} 

\newcommand{\lip}{L}  
\newcommand{\sm}{\beta} 
\renewcommand{\sc}{\lambda} 
\newcommand{\rad}{D} 

\newcommand{\hacomment}[1]{
		\textcolor{blue}{\textbf{HA:} {#1}}
}

\title{Private Stochastic Convex Optimization: \\ Optimal Rates in $\ell_1$ Geometry}
\author{%
    Hilal Asi\thanks{Stanford University, part of this work performed while interning at Apple; \texttt{asi@stanford.edu}.}
    \and Vitaly Feldman\thanks{Apple; \texttt{vitaly.edu@gmail.com}.}
    \and Tomer Koren\thanks{School of Computer Science, Tel Aviv University, and Google; \texttt{tkoren@tauex.tau.ac.il}.}
    \and Kunal Talwar\thanks{Apple; \texttt{kunal@kunaltalwar.org}.}
    }

\begin{document}
\maketitle

\begin{abstract}%
Stochastic convex optimization over an $\ell_1$-bounded domain is ubiquitous in machine learning applications  such  as  LASSO but  remains poorly understood when learning with differential privacy. We show that, up to logarithmic factors the optimal excess population loss of any $(\eps,\delta)$-differentially private optimizer is 
$\sqrt{\log(d)/n} + \sqrt{d}/\eps n.$ 
The upper bound is based on a new algorithm that combines the iterative localization approach of~\citet{FeldmanKoTa20} with a new analysis of private regularized mirror descent. It applies to $\ell_p$ bounded domains for $p\in [1,2]$ and queries at most $n^{3/2}$ gradients improving over the best previously known algorithm for the $\ell_2$ case which needs $n^2$ gradients. 
Further, we show that when the loss functions satisfy additional smoothness assumptions, the excess loss is upper bounded (up to logarithmic factors) by
$\sqrt{\log(d)/n} + (\log(d)/\eps n)^{2/3}.$ 
This bound is achieved by a new variance-reduced version of the Frank-Wolfe algorithm that requires just a single pass over the data. We also show that the lower bound in this case is the minimum of the two rates mentioned above.
\end{abstract}

\section{Introduction}
\label{sec:intro}


Convex optimization is one of the most well-studied problems in private data analysis. Existing works have largely studied optimization problems over $\ell_2$-bounded domains. However several machine learning applications, such as LASSO and minimization over the probability simplex, involve optimization over $\ell_1$-bounded domains. In this work we study the problem of differentially private stochastic convex optimization (DP-SCO) over $\ell_1$-bounded domains.
 

In this problem (DP-SCO),
given $n$ i.i.d. samples $z_1,\dots,z_n$ from a distribution $P$, we wish to release a private solution $x \in \xdomain \subseteq \R^d$ that minimizes the population loss $F(x) = \E_{z \sim P}[f(x;z)]$ for a convex function $f$ over $x$. The algorithm's performance is measured using the excess population loss of the solution $x$, that is $F(x) - \min_{y \in \xdomain} F(y)$. The optimal algorithms and rates for this problem---even without privacy---have a crucial dependence on the geometry of the constraint set $\xdomain$ and in this work we focus on sets with bounded $\ell_1$-diameter. Without privacy constraints, there exist standard and efficient algorithms, such as mirror descent and exponentiated gradient decent, that achieve the optimal excess loss $O(\sqrt{\log(d)/n})$~\cite{ShalevBe14}. The landscape of the problem, however, with privacy constraints is not fully understood yet.

Most prior work on private convex optimization has focused on minimization of the empirical loss $\hat F(x) = \frac{1}{n} \sum_{i=1}^n f(x;z_i)$ over $\ell_2$-bounded domains~\cite{ChaudhuriMoSa11,BassilySmTh14,BassilyFeTaTh19}. \citet{BassilySmTh14} show that the optimal excess empirical loss in this setting is $\Theta({\sqrt{d}}/{\diffp n})$ up to log factors. More recently, \citet{BassilyFeTaTh19} give an asymptotically tight bound of $1/\sqrt{n} + \sqrt{d}/(\eps n)$ on the excess population loss in this setting using noisy gradient descent. Under mild smoothness assumptions, \citet{FeldmanKoTa20} develop algorithms that achieve the optimal excess population loss using $n$ gradient computations.

In contrast, existing results for private optimization in $\ell_1$-geometry do not achieve the optimal rates for the excess population loss~\cite{KiferSmTh12,JainTh14,TalwarThZh15}.
For the empirical loss, \citet{TalwarThZh15} develop private algorithms with $\wt O(1/(n \diffp)^{2/3})$ excess empirical loss for smooth functions and provide tight lower bounds when the dimension $d$ is sufficiently high. These bounds can be converted into bounds on the excess population loss using standard techniques of uniform convergence of empirical loss to population loss, however these techniques can lead to suboptimal bounds as there are settings where uniform convergence is lower bounded by $\Omega(\sqrt{d/n})$~\citep{Feldman16}. 
Moreover, the algorithm of~\citet{TalwarThZh15} has runtime $O(n^{5/3})$ in the moderate privacy regime ($\diffp = \Theta(1)$) which is prohibitive in practice.
On the other hand, 
\citet{JainTh14} develop algorithms for the population loss, however, their work is limited to generalized linear models and achieves a sub-optimal rate $\wt O(1/n^{1/3})$.

In this work we develop private algorithms that achieve the optimal excess population loss in $\ell_1$-geometry, demonstrating that significant improvements are possible when the functions are smooth, in contrast to $\ell_2$-geometry where smoothness does not lead to better bounds. Specifically, for non-smooth functions, we develop an iterative localization algorithm,  based on noisy mirror descent which achieves the optimal rate $\sqrt{\log(d)/n} + \sqrt{d}/\eps n$. With additional smoothness assumptions,  we show that rates with logarithmic dependence on the dimension are possible using a private variance-reduced Frank-Wolfe algorithm which obtains the rate $\sqrt{\log(d)/n} + (\log(d)/\eps n)^{2/3}$ and runs in linear (in $n$) time. This shows that privacy is essentially free in this setting even when $d \gg n$ and $\diffp$ is as small as $n^{-1/4}$.
Finally, we show that similar rates are possible for general $\ell_p$-geometries for non-smooth functions when $1 \le p \le 2$.
Moreover, our algorithms query at most 
$O(n^{3/2}) $ gradients which improves over the best known algorithms for the non-smooth case in $\ell_2$-geometry which require $n^2$ gradients~\cite{FeldmanKoTa20}.

The following two theorems summarize our upper bounds.
\begin{theorem}[non-smooth functions]
\label{thm:non-smooth}
    Let $\xdomain \subset \R^d$ be a convex body with $\ell_1$ diameter less than $1$.
    Let $f(\cdot;\ds)$ be convex, Lipschitz
    with respect to $\lone{\cdot}$ for any $\ds \in \domain$.
    There is an $(\diffp,\delta)$-DP algorithm that takes a dataset $\Ds \in \domain^n$, queries at most $O(\log n \cdot \min(n^{3/2} \sqrt{\log d}, n^2 \diffp/\sqrt{d})) $ and outputs a solution $\hat x$ that has
    \begin{equation*}
    \E[F(\hat x)] \le \min_{x\in \xdomain} F(x) + 
             \wt O \left( \sqrt{\frac{\log d}{n}}
            + \frac{ \sqrt{d} \log^{3/2} d}{n \diffp} \right)
            ,
    \end{equation*}
    where the expectation is over the random choice of $\Ds$ and the randomness of the algorithm.
\end{theorem}
\begin{theorem}[smooth functions]
\label{thm:smooth}
    Let $\xdomain = \{ x \in \R^d: \lone{x} \le 1 \}$ be the $\ell_1$-ball. 
    Let $f(\cdot;z)$ be convex, Lipschitz and smooth
    with respect to $\lone{\cdot}$ for any $z \in \Z$.
    There is an $(\diffp,\delta)$-DP linear time algorithm that takes a dataset $\Ds \in \domain^n$ and outputs a solution $\hat x$ that has
    \begin{equation*}
    \E[F(x_{K}) ] 
        \le \min_{x\in \xdomain} F(x) + \wt O \left( 
        \sqrt{\frac{\log d}{n}}
        + \left( 
        \frac{\log d}{n\diffp} \right)^{2/3} 
        \right),
    \end{equation*}
    where the expectation is over the random choice of $\Ds$ and the randomness of the algorithm. \\
\end{theorem}
Before proceeding to review our algorithmic techniques, we briefly explain why the approaches used to obtain optimal rates in $\ell_2$-geometry~\cite{BassilyFeTaTh19,FeldmanKoTa20} do not work in our setting.
One of the most natural approaches to proving bounds for private stochastic optimization is to use the generalization properties of differential privacy to derive population loss bounds for a private ERM algorithm. This approach fails to give asymptotically optimal bounds for the $\ell_2$ case~\citep{BassilySmTh14}, and similarly gives suboptimal bounds for the $\ell_1$ case. Broadly, there are two approaches that have been used to get optimal bounds in the $\ell_2$ case. An approach due to~\citet{BassilyFeTaTh19} uses stability of SGD on sufficiently smooth losses~\citep{HardtRS16} to get population loss bounds.  These stability results rely on contractivity of gradient descent steps. However, as we show in an example that appears in~\cref{sec:MD-non-contractive}, the versions of mirror descent that are relevant to our setting do not have this property. 
\citet{FeldmanKoTa20} derive generalization properties of their one pass algorithms from online-to-batch conversion. However, their analysis still relies on contractivity to prove the privacy guarantees of their algorithm.
For their iterative localization approach~\citet{FeldmanKoTa20} use stability of the optimal solution to ERM in a different way to determine the scale of the noise added in each phase of the algorithm. In $\ell_1$ geometry the norm of the noise added via this approach would overwhelm the signal (we discuss this in detail below). 


We overview the key techniques we use to overcome these challenges below.

\paragraph{Mirror descent based Iterative Localization.} 

In the non-smooth setting, we build on the iterative localization framework of~\citet{FeldmanKoTa20}.
In this framework in each phase a non-private optimization algorithm is used to 
solve a regularized version of the optimization problem. Regularization ensures that the output solution has small sensitivity and thus addition of Gaussian noise guarantees privacy. By appropriately choosing the noise and regularization scales, each phase reduces the distance to an approximate minimizer by a multiplicative factor. Thus after a logarithmic number of phases, the current iterate has the desired guarantees. Unfortunately, addition of Gaussian noise (and other output perturbation techniques)  results in sub-optimal bounds in $\ell_1$-geometry since the $\ell_1$-error due to noise grows linearly with $d$. In contrast, the $\ell_2$-error grows as~$\sqrt{d}$.

Instead of using output perturbation, we propose to use a private optimization algorithm in each phase. Using stability properties of strongly convex functions, we show that if the output of the private algorithm has sufficiently small empirical excess loss, then it has to be close to an approximate minimizer. Specifically, we reduce the distance to a minimizer by a multiplicative factor (relative to the initial conditions at that phase). We show that a private version of mirror descent for strongly convex empirical risk minimization achieves sufficiently small excess empirical loss giving us an algorithm that achieves the optimal rate for non-smooth loss functions. More generally, this technique reduces the problem of DP-SCO to the problem of DP-ERM with strongly convex objectives. We provide details and analysis of this approach in~\Cref{sec:non-smooth}.




\paragraph{Dyadic variance-reduced Frank-Wolfe.} 

Our second algorithm is based on recent progress in stochastic optimization. \citet{YurtseverSrCe19} developed (non-private) variance-reduced Frank-Wolfe algorithm that achieves the optimal $\wt O(1/\sqrt{n})$ excess population loss improving on the  standard implementations of Frank-Wolfe that achieve excess population loss of $\wt O(1/n^{1/3})$.  The improvement relies on a novel variance reduction techniques that uses previous samples to improve the gradient estimates at future iterates~\cite{FangChLiLiZh18}. 
This frequent reuse of samples is the main challenge in developing a private version of the algorithm.

Inspired by the binary tree technique in the privacy literature~\cite{DworkNaPiRo10,DworkNaReRo15}, we develop a new binary-tree-based variance reduction technique for the Frank-Wolfe algorithm.
At a high level, the algorithm constructs a binary tree and allocates a set of samples to each vertex. The gradient at each vertex is then estimated using the samples of that vertex and the gradients along the path to the root. We assign more samples (larger batch sizes) to vertices that are closer to the root, to account for the fact that they are reused in more steps of the algorithm. This ensures that the privacy budget of samples in any vertex is not exceeded.

Using this privacy-aware design of variance-reduction, we rely on two tools to develop and analyze our algorithm.
First, similarly to the private Frank-Wolfe for ERM~\citep{TalwarThZh15}, we use the exponential mechanism to privatize the updates. A Frank-Wolfe update chooses one of the vertices of the constraint set ($2d$ possibilities including signs for $\ell_1$-balls) and therefore the application of the exponential mechanism leads to a logarithmic dependence on the dimension $d$. 
This tool together with the careful accounting of privacy losses across the nodes, suffices to get the optimal bounds for the pure $\diffp$-DP case ($\delta = 0$). To get the optimal rates for $(\diffp,\delta)$-DP, we rely on recent amplification by shuffling result for private local randomizers~\cite{FeldmanMcTa20}. To amplify privacy, we view our algorithm as a sequence of local randomizers, each operating on a different subset of the tree. 
\Cref{sec:FW} contains details of this algorithm.

In independent and concurrent work, \citet{BassilyGuNa21} study differentially private algorithms for stochastic optimization in $\ell_p$-geometry. Similarly to our work, they build on mirror descent and variance-reduced Frank-Wolfe algorithms to design private procedures for DP-SCO albeit without the iterative localization scheme and the binary-tree-based sample allocation technique we propose. As a result, their algorithms achieve sub-optimal rates in some of the parameter regimes: in $\ell_1$-geometry, they achieve excess loss of roughly $\log(d)/\diffp \sqrt{n}$ in contrast to the $\sqrt{\log(d)}/\sqrt{n} + \log(d)/(\diffp n)^{2/3}$ rate of our algorithms. For $1 < p <2$, their algorithms have excess loss of (up to log factors) $\min(d^{1/4}/\sqrt{n}, \sqrt{d}/ (\diffp n^{3/4}))$, whereas our algorithms achieve the rate of $\sqrt{d}/\diffp n$. On the other hand, \citet{BassilyGuNa21} develop a generalized Gaussian mechanism for adding noise in $\ell_p$-geometry. Their mechanism improves over the standard Gaussian mechanism and can improve the rates of our algorithms for $\ell_p$-geometry (\Cref{thm:local-MD-general-geom}) by a $\sqrt{\log d}$ factor. Moreover, they prove a lower bound for $\ell_p$-geometries with $1 < p < 2$ that establishes the optimality of our upper bounds for $1 < p < 2$.

\section{Preliminaries}
\label{sec:pre}

\subsection{Stochastic Convex Optimization}
We let $\Ds = (\ds_1,\dots,\ds_n)$ denote datasets where $\ds_i \in \domain$ are drawn i.i.d. from a distribution $P$ over the domain $\domain$. Let $\xdomain \subseteq \R^d$ be a convex set that denotes the set of parameter for the optimization problem. Given a loss function $f(x;\ds) : \xdomain \times \domain \to \R$ that is convex in $x$ (for every $\ds$), we define the population loss $F(x) = \E_{\ds \sim P} [f(x;\ds)]$. The excess population loss of a parameter $x \in \xdomain$ is then $F(x) - \min_{y \in \xdomain} F(y)$. We also consider the empirical loss $\hat F(x;S) = \frac{1}{n} \sum_{i=1}^n f(x;\ds_i)$ and the excess empirical loss of $x \in \xdomain$ is $\hat F(x;S) - \min_{y \in \xdomain} \hat F(y;S)$. For a set $\xdomain$, we will denote its $\ell_p$ diameter by $\diam_p(\xdomain) = \sup_{x,y \in \xdomain} \norm{x-y}_p$.

As we are interested in general geometries, we define the standard properties (e.g., Lipschitz, smooth and strongly convex) with respect to a general norm which are frequently used in the optimization literature~\cite{Duchi18}.

\begin{definition}[Lipschitz continuity]
    A function $f: \xdomain \to \R$ is $\lip$-Lipschitz with respect to a norm $\norm{\cdot}$ over $\xdomain$ if for every $x,y \in \xdomain$ we have $|f(x) - f(y)| \le \lip \norm{x-y}$.
\end{definition}

A standard result is that $\lip$-Lipschitz continuity is equivalent to bounded (sub)-gradients, namely that $\dnorm{g} \le \lip$ for all $x \in \xdomain$ and sub-gradient $g\in \partial f(x)$ where $\dnorm{\cdot}$ is the dual norm of $\norm{\cdot}$.

\begin{definition}[smoothness]
    A function $f: \xdomain \to \R$ is $\sm$-smooth with respect to a norm $\norm{\cdot}$ over $\xdomain$ if for every $x,y \in \xdomain$ we have $\dnorm{\nabla f(x) - \nabla f(y)} \le \sm \norm{x-y}$.
\end{definition}

\begin{definition}[strong convexity]
    A function $f: \xdomain \to \R$ is $\sc$-strongly convex with respect to a norm $\norm{\cdot}$ over $\xdomain$ if for any $x,y \in \xdomain$ we 
    have 
       $f(x) + \<\nabla f(x), y-x \> + \frac{\sc}{2} \norm{y-x}^2 \le f(y)$. 
\end{definition}

Since we develop private versions of mirror descent, we define the Bregman divergence associated with a differentiable convex function $h: \xdomain \to \R$ to be
      $ \db(x,y) = h(x) - h(y) - \<\nabla h (y), x-y\>$. 
We require a definition of strong convexity relative to a function which has been used in several works in the optimization literature~\cite{DuchiShSiTe10,LuFrNe18}.
\begin{definition}[relative strong convexity]
    A function $f: \xdomain \to \R$ is $\sc$-strongly convex
    relative to $h: \xdomain \to \R$ if for any $x,y \in \xdomain$,
      $  f(x) + \<\nabla f(x), y-x \> + \sc \db(y,x) \le f(y)$.
\end{definition}
Note that if $h(x)$ is convex, then $h(x)$ is 
$1$-strongly convex relative to $h(x)$ according to 
this definition. Moreover, the function $f(x) = g(x) + h(x)$ 
is also $1$-strongly convex relative to $h(x)$ for 
any convex function $g(x)$.


\subsection{Differential Privacy}
We recall the definition of \ed-differential privacy.
\begin{definition}[\citealp{DworkMcNiSm06,DworkKeMcMiNa06}]
\label{def:DP}
	A randomized  algorithm $\A$ is \ed-differentially private (\ed-DP) if, for all datasets $\Ds,\Ds' \in \domain^n$ that differ in a single data element and for all events $\cO$ in the output space of $\A$, we have
	\[
	\Pr[\A(\Ds)\in \cO] \leq e^{\eps} \Pr[\A(\Ds')\in \cO] +\delta.
	\]
\end{definition}
\iftoggle{arxiv}{
To simplify notation, we sometimes use the notion of \ed-indistinguishability; 
two random variables $X$ and $Y$ are \ed-indistinguishable, denoted $X \approx_{(\eps,\delta)} Y$, if for every $\cO$, $\Pr(X \in \cO) \le e^{\eps} \Pr[Y \in \cO] +\delta$ and $\Pr(Y \in \cO) \le e^{\eps} \Pr[X \in \cO] +\delta$. }{}
When $\delta=0$, we use the shorter notation $\eps$-DP. 
We also use the following privacy composition results.

\begin{lemma}[Basic composition~\citealp{DworkRo14}]
\label{lemma:basic-comp}    
    If $\A_1,\dots,A_k$ are randomized algorithms that each is $\diffp$-DP, then their composition $(\A_1(\Ds),\dots,A_k(\Ds))$ is $k \diffp$-DP.
\end{lemma}

\begin{lemma}[Advanced composition~\citealp{DworkRo14}] 
\label{lemma:advanced-comp}
    If $\A_1,\dots,A_k$ are randomized algorithms that each is $(\diffp,\delta)$-DP, then their composition $(\A_1(\Ds),\dots,A_k(\Ds))$ is $(\sqrt{2k \log(1/\delta')} \diffp + k \diffp (e^\diffp - 1),\delta' + k \delta)$-DP.
\end{lemma}

\section{Algorithms for Non-Smooth Functions}
\label{sec:non-smooth}

In this section, we develop an algorithm that builds 
on the iterative localization techniques of~\citet{FeldmanKoTa20} to
achieve optimal excess population loss for non-smooth functions over the $\ell_1$-ball. Instead of using output perturbation to solve the regularized optimization problems, our algorithm uses general
private algorithms for solving strongly convex ERM problems.  
This essentially reduces the problem of privately minimizing the population loss to that of privately minimizing a strongly convex empirical risk. In Section~\ref{sec:MD-ERM} we develop private versions of mirror descent that achieve optimal bounds for strongly convex ERM problems, and in Section~\ref{sec:MD-pop} we use these algorithms in an iterative localization framework to obtain optimal bounds for the population loss.


\subsection{Private Algorithms for Strongly Convex ERM}
\label{sec:MD-ERM}
In this section, we consider empirical risk minimization for strongly convex functions and achieve optimal excess empirical loss using  noisy mirror descent (\cref{alg:noisy-MD}).

\begin{algorithm}
	\caption{Noisy Mirror Descent}
	\label{alg:noisy-MD}
	\begin{algorithmic}[1]
		\REQUIRE Dataset $\Ds=(\ds_1, \ldots, \ds_n)\in \domain^n$,
		convex set $\xdomain$,
		convex function $h: \xdomain \to \R$,
		step sizes $\{ \ss_k \}_{k=1}^T$, 
		batch size $b$, 
		initial point $x_0$,
		number of iterations $T$;
				
        \FOR{$k=1$ to $T$\,}
        	\STATE Sample $S_1,\dots,S_b\sim\mbox{Unif}(\Ds)$ 
        	\STATE Set $\hat g_k = \frac{1}{b} \sum_{i=1}^b \nabla f(x_k; S_i) + \noise_i$
        	where $\noise_i \sim \normal(0,\sigma^2 I_d)$
        	with $ \sigma =  {100 \lip\sqrt{d \log(1/\delta)}}/{ b \diffp}$
        	\STATE Find 
        	$x_{k+1} \defeq \argmin_{x \in \mc{X}} 
        	\{\<\hat g_k, x - x_k\> + \frac{1}{\ss_k} \db(x,x_k)  \}$
            \ENDFOR
            \RETURN $\bar{x}_T=\frac{1}{T}\sum_{k=1}^{T} x_k$ (convex)
            \RETURN $\hat{x}_T=\frac{2}{T(T+1)}\sum_{k=1}^{T} k x_k$ (strongly convex) 
	\end{algorithmic}
	\label{Alg:NSGD}
\end{algorithm}

\begin{theorem}
\label{thm:noisy-MD}
    Let $h: \xdomain \to \R$ be $1$-strongly convex with respect 
    to $\lone{\cdot}$, $x\opt = \argmin_{x \in \xdomain} \hat F(x;S)$, and assume $\db(x\opt,x_0) \le \rad^2$.
    Let $f(x;z)$ be convex and $\lip$-Lipschitz 
    with respect to $\lone{\cdot}$ for all $z \in \domain$.
    Setting $1 \le b$,
    $T = \frac{n^2}{b^2}$ and $\ss_k = \frac{D}{ \sqrt{T}} \frac{1}{\sqrt{\lip^2 + 2 \sigma^2 \log d}}$,
    \cref{alg:noisy-MD} is $(\diffp,\delta)$-DP 
    and 
    \begin{equation*}
        \E[\hat F(\bar x_T;S) - \hat F(x\opt;S)]
            \le \lip \rad  \cdot O \Bigg( \frac{b}{n}
                + \frac{\sqrt{d \log d \log\tfrac{1}{\delta}} }{n \diffp} \Bigg).
    \end{equation*}
    Moreover, if $f(x;z)$ is $\sc$-strongly convex relative to $h(x)$, then
    setting $\ss_k = \frac{2}{\sc (k+1)}$
    \begin{align*}
        \E[\hat F(\hat x_T;S) - \hat F(x\opt;S)]
            & \le O \left( \frac{\lip^2 b^2}{\sc n^2} + \frac{\lip^2 d \log d \log \tfrac{1}{\delta} }{\sc n^2 \diffp^2} \right).
    \end{align*}
\end{theorem}

To prove~\cref{thm:noisy-MD}, we need the following
standard results for the convergence of stochastic
mirror descent for convex and strongly convex functions.

\begin{lemma}[\citealp{Duchi18}, Corollary 4.2.11]
\label{lemma:conv-MD}
    Assume $h(x)$ is 1-strongly convex with respect 
    to $\lone{\cdot}$.
    Let $f(x)$ be a convex function and
    $x\opt = \argmin_{x \in \mc{X}} f(x)$.
    Consider the 
    stochastic mirror descent update
    $x_{k+1} =  \argmin_{x \in \mc{X}} 
        	\{\< g_k, x - x_k\> + \frac{1}{\ss_k} \db(x,x_k)  \}$
    where $\E[g_k] \in \partial f(x_k)$ with
    $\E\left[\linf{g_k}^2 \right] \le \lip^2$.
    If $\ss_k = \ss$ for all $k$ then
    the average iterate $\bar x_T = \frac{1}{T} \sum_{i=1}^T x_i$ has
       $\E[f(\bar x_T) - f(x\opt)]
            \le \frac{\db(x\opt,x_1)}{T \ss} + \frac{\ss \lip^2 }{2}$.
\end{lemma}

We also need the following result which states the rates of stochastic mirror descent for strongly convex functions. 
Similar results appear in the optimization literature~\cite{lacoste2012simpler}, though as the statement we require is less common, we provide a proof in~\cref{sec:proof-conv-MD-sc}.

\begin{lemma}
\label{lemma:conv-MD-sc}
    Under the same notation of~\cref{lemma:conv-MD},
    if $f(x)$ is $\sc$-strongly convex relative
    to $h(x)$, then setting $\ss_k = \frac{2}{\sc (k+1)}$ the weighted average $\hat x_T = \frac{2}{T(T+1)} \sum_{k=1}^T k  x_k$ has
    $
        \E[f(\hat x_T) - f(x\opt)]
            \le \frac{\lip^2}{\sc (T+1)}.
    $
\end{lemma}

We are now ready to prove~\cref{thm:noisy-MD}.

\begin{proof}
    \iftoggle{arxiv}{
    The privacy proof follows directly using Moments accountant, that is, Theorem 1 in \cite{AbadiChGoMcMiTaZh16}, by noting the the $\ell_2$-norm of the gradients is bounded by $\ltwo{\nabla f(x; z_i)} \le \linf{\nabla f(x; z_i)} \sqrt{d}
    \le \lip \sqrt{d}$ for all $x \in \xdomain$ and $z \in \domain$.
    }
    {
    The privacy guarantees follow from standard properties of the Gaussian mechanism and Moments accountant~\cite{AbadiChGoMcMiTaZh16}; we provide full details in~\cref{sec:proofs-non-smooth}.
    }
    Now we prove the utility
    of the algorithm. To this end, we have that
    $\E[\linf{\hat g_k}^2] \le 2 \lip^2 + 2 \E[\linf{\noise_k}^2] \le 2 \lip^2 + 4 \sigma^2 \log d$. \cref{lemma:conv-MD} now implies that
    \begin{align*}
     \E[\hat F(\bar x_T;S) - \hat F(x\opt;S)]
            & \le \frac{\rad^2}{T \ss} + \ss \lip^2  
                + 2 \ss \sigma^2 \log d \\
            & \le  {2\rad \sqrt{(\lip^2 + 2 \sigma^2 \log d) /T}} \\
            & \le \lip \rad \cdot O \left(\frac{b}{n}
                + \frac{\sqrt{d \log d \log \tfrac{1}{\delta}} }{n \diffp}\right) ,
    \end{align*}
    where the second inequality follows from the choice of $\ss$.
    For the second part, \cref{lemma:conv-MD-sc} implies that
    \begin{align*}
        \E[\hat F(\hat x_T;S) - \hat F(x\opt;S)]
            & \le \frac{\lip^2}{\sc} O \left( \frac{ b^2}{n^2} + \frac{d \log d \log \tfrac{1}{\delta}}{n^2 \diffp^2} \right)
            \!.
            \qedhere
    \end{align*}
\end{proof}

\subsection{Private Algorithms for SCO}
\label{sec:MD-pop}
Building on the noisy mirror descent
algorithm 
of Section~\ref{sec:MD-ERM}, in this section we develop a localization based algorithm for
the population loss 
that achieves the optimal bounds in $\ell_1$ geometry. 
The algorithm iteratively solves a regularized version of the (empirical) objective function using noisy mirror decent (\cref{alg:noisy-MD}). We present the full details in~\cref{alg:phased-erm-MD} which enjoys the following guarantees.


\begin{algorithm}
	\caption{Localized Noisy Mirror Descent}
	\label{alg:phased-erm-MD}
	\begin{algorithmic}[1]
		\REQUIRE Dataset $\Ds=(\ds_1, \ldots, \ds_n)\in \domain^n$,
		constraint set $\xdomain$,
		step size $\ss$, initial point $x_0$;
		\STATE Set $k = \ceil{\log n}$, 
		       $p = 1 + 1/\log d$	
        \FOR{$i=1$ to $k$\,}
        	\STATE Set $\diffp_i = 2^{-i} \diffp$, 
        	        $n_i= 2^{-i} n$, 
        	        $\ss_i = 2^{-4i} \ss$ 
        	\STATE Apply~\cref{alg:noisy-MD} with $(\diffp_i,\delta)$-DP, batch size $b_i =\max(\sqrt{n_i/\log d}, \sqrt{d/\diffp_i})$, $T = n_i^2/b_i^2$ and $h_i(x) = \frac{1}{p-1} \norm{x - x_{i-1}}^2_p$
        	for solving the ERM over 
        	$ \mc{X}_i= \{x\in \xdomain: \norm{x - x_{i-1}}_p \le {2\lip \ss_i n_i (p-1)} \}$:\\
        	$
        	    \displaystyle
        	    F_i(x) =  \frac{1}{n_i} \sum_{j=1}^{n_i} f(x;z_j) + \frac{1}{\ss_i n_i (p-1)} \norm{x - x_{i-1}}_p^2
        	$
        	\STATE Let $x_i$ be the output of the private
        	algorithm
            \ENDFOR
            \RETURN the final iterate $x_k$
	\end{algorithmic}
\end{algorithm}


\begin{theorem}
\label{thm:local-MD}
    Assume $\diam_1(\mc{X}) \le \rad$ and 
    $f(x;z)$ is convex and $\lip$-Lipschitz with respect to 
    $\lone{\cdot}$ for all $z \in \domain$.
    If we set 
    \begin{equation*}
        \ss = \frac{\rad}{\lip} \min 
        \left\{ {\sqrt{\log(d)/n}}, {\diffp}/{\sqrt{d \log d \log \tfrac{1}{\delta}}} \right\},
    \end{equation*}
    then~\cref{alg:phased-erm-MD} uses $O(\log n \cdot \min(n^{3/2} \sqrt{\log d}, n^2 \diffp/\sqrt{d})) $ gradients and its output has
    \begin{equation*}
       \E[F(x_k) - F(x\opt)] 
       = \lip \rad  \cdot O \Bigg( \frac{\sqrt{\log d}}{\sqrt{n}}
            + \frac{ \sqrt{d \log^3 d \log \tfrac{1}{\delta} }}{n \diffp} \Bigg). 
    \end{equation*}
\end{theorem}

We begin with the following lemma which bounds the distance
of the private minimizer to the true minimizer at 
each iteration.
\begin{lemma}
\label{lemma:opt-dis-MD}
    Let $\hat x_i = \argmin_{x \in \mc{X}} F_i(x)$.
    Then ,
    \begin{equation*}
        \E [\norm{x_i - \hat x_i}_p^2]
    \le O \left(\frac{\lip^2 \ss_i^2 n_i}{\log d} +  {\lip^2 \ss_i^2 d \log d \log(1/\delta)}/{\diffp_i^2} \right).  
    \end{equation*}
\end{lemma}
\iftoggle{arxiv}{
\begin{proof}
    First, we prove that $\hat x_i \in \mc{X}_i$.
    The definition of $\hat x_i$ implies that 
    \begin{equation*}
        \frac{1}{n_i} \sum_{j=1}^{n_i} f(\hat x_i;z_j) + \frac{1}{\ss_i n_i (p-1)} \norm{\hat x_i - x_{i-1}}_p^2
        \le \frac{1}{n_i} \sum_{j=1}^{n_i} f(x_{i-1};z_j).
    \end{equation*}
    Since $f(x;z)$ is $\lip$-Lipschitz, we get
    \begin{equation*}
        \frac{1}{\ss_i n_i (p-1)} \norm{\hat x_i - x_{i-1}}_p^2
        \le \lip \lone{\hat x_i - x_{i-1}} 
        \le 2 \lip \norm{\hat x_i - x_{i-1}}_p
    \end{equation*}
    where the last inequality follows from the choice of $p$ 
    (since $\lone{z} \le d^{1 - 1/p} \norm{z}_p \le 2 \norm{z}_p$ for all
    $z \in \R^d$),
    hence we get $\norm{\hat x_i - x_{i-1}}_p \le \frac{2\lip \ss_i n_i}{\log d}$.
    Thus, we have that $\hat x_i \in \mc{X}_i
    = \{x: \norm{x - x_{i-1}}_p \le \frac{2\lip \ss_i n_i}{\log d}\}$.
    
    Now, note that the function $F_i(x)$ 
    is $\sc_i$-strongly convex relative to 
    $h_i(x) = \frac{1}{p-1} \norm{x - x_{i-1}}_p^2$ where
    $\sc_i = \frac{1}{\ss_i n_i}$.
    Moreover, the function $r_i(x) = \frac{1}{\ss_i n_i (p-1)} \norm{x - x_{i-1}}_p^2$ is $4 \lip$-Lipschitz with 
    respect to $\lone{\cdot}$ for 
    $x \in \mc{X}_i$. 
    Therefore using the bounds of~\cref{thm:noisy-MD}
    for noisy mirror descent and observing that 
    $F_i(x)$ is $\sc_i$-strongly convex 
    with respect to $\norm{\cdot}_p$,
    \begin{equation*}
        \frac{\sc_i}{2} \E [\norm{x_i - \hat x_i}_p^2]
            \le \E[F_i(x_i) - F_i(\hat x_i)]
            \le O \left( \frac{\lip^2}{\sc_i n_i \log d} + \frac{\lip^2 d \log d \log^2(1/\delta)}{n_i^2 \diffp_i^2 \sc_i} \right),
    \end{equation*}
    implying that $\E [\norm{x_i - \hat x_i}_p^2]
    \le O \left( \frac{\lip^2 \ss_i^2 n_i}{\log d} + \frac{\lip^2 \ss_i^2 d \log d \log^2(1/\delta)}{\diffp_i^2} \right)$.
\end{proof}
}

The next lemma follows from~\citet{ShwartzShSrSr09}.
\begin{lemma}
\label{lemma:fun-subopt-MD}
    Let $\hat x_i = \argmin_{x \in \mc{X}_i} F_i(x)$ and 
    $y \in \mc{X}$. If $f(x;z)$ is $\lip$-Lipschitz with respect
    to $\lone{\cdot}$, then
        $\E[F(\hat x_i)] - F(y) 
        \le \frac{\E[\norm{y - x_{i-1}}_p^2] }{\ss_i n_i (p-1)} 
            + O( \lip^2 \ss_i)$.    
\end{lemma}
\iftoggle{arxiv}{
\begin{proof}
    The proof follows from Theorems 6 and 7 in~\cite{ShwartzShSrSr09} by noting that the function
    $r(x;z_j) = f(x;z_j) + \frac{1}{\ss_i n_i (p-1)} \norm{x - x_{i-1}}_p^2$ is $\frac{1}{\ss_i n_i}$-strongly convex
    and $O(\lip)$-Lipschitz with respect to $\lone{\cdot}$ 
    over $\mc{X}_i$.
\end{proof}
}

We are now ready to prove~\cref{thm:local-MD}.
\begin{proof}
    First, we prove the claim about runtime and number of queried gradients. \cref{alg:noisy-MD} requires $n_i^2/b_i$ gradients (same runtime) hence since $b_i =\max(\sqrt{n_i/\log d}, \sqrt{d/\diffp_i})$ we get that the number of gradients at each stage is at most $\min(n^{3/2} \sqrt{\log d}, n^2 \diffp/\sqrt{d})$, implying the claim as we have $\log n$ iterates.
    Next, we prove utility which is similar to the proof of Theorem 4.4
    in~\cite{FeldmanKoTa20}. 
    Letting $\hat x_0 = x\opt$, we have:
    \begin{align*}
        \iftoggle{arxiv}{}{&} \E[F(x_k)] - F(x\opt)
        \iftoggle{arxiv}{}{\\}
        & = \sum_{i=1}^k \E[F(\hat x_i) - F(\hat x_{i-1})]
            + \E[F(x_k) - F(\hat x_k)].
    \end{align*}
    First, note that~\cref{lemma:opt-dis-MD} implies
     \begin{align*}
        \iftoggle{arxiv}{}{&} \E[F(x_k) - F(\hat x_k)]
        \iftoggle{arxiv}{}{\\}
        & \le \lip \E[ \lone{x_k - \hat x_k} ] \\
        & \le \lip \sqrt{\E[ 2 \norm{x_k - \hat x_k}_p^2 ]} \\
        & \le C {\lip^2 \ss_k ( \sqrt{n_i/ \log d} + \sqrt{d \log d \log(1/\delta)}}/{\diffp_k}) \\
        & \le C {2^{-2k} \lip^2 \ss (\sqrt{n/\log d} + \sqrt{d \log d \log(1/\delta)} }/{\diffp} ) 
         \le C {\lip \rad}/{n^2},
    \end{align*}
    where the last inequality follows since
    $\ss \le \frac{\rad}{\lip} \min (\sqrt{\log(d)/n}, {\diffp}/{\sqrt{d \log d \log(1/\delta)}})$. Lemmas~\ref{lemma:fun-subopt-MD}
    and~\ref{lemma:opt-dis-MD} imply
     \begin{align*}
    \iftoggle{arxiv}{}{&} \sum_{i=1}^k \E[F(\hat x_i) - F(\hat x_{i-1})] \iftoggle{arxiv}{}{\\}
            & \le \sum_{i=1}^k \frac{\E[\norm{\hat x_{i-1} - x_{i-1}}_p^2] }{\ss_i n_i (p-1)} + C \lip^2 \ss_i \\
            & \le \frac{\rad^2}{\ss n (p-1)} 
                + \sum_{i=2}^k C ( {\lip^2 \ss_i} + \frac{\lip^2 \ss_i d \log d \log(1/\delta)}{n_i \diffp_i^2 (p-1)})  
                + C \sum_{i=1}^k \frac{\lip^2 \ss}{2^{i}} \\
            & \le \frac{\rad^2 }{\ss n (p-1)} 
                + C \lip^2 \ss + C \sum_{i=2}^k 2^{-i} \frac{\lip^2 \ss d \log d \log(1/\delta)}{n \diffp^2 (p-1)}  
                + 2C \lip^2 \ss \\
            & \le \frac{\rad^2}{\ss n (p-1)} 
                 + 2 C  \frac{\lip^2 \ss d \log d \log(1/\delta)}{n \diffp^2 (p-1)}  
                + 3C \lip^2 \ss.
    \end{align*}
    The claim now follows by setting the value of $\ss$.
\end{proof}

Finally, we can extend~\Cref{alg:phased-erm-MD} to work for general $\ell_p$ geometries for $1 < p \le 2$, resulting in the following theorem. We defer full details to~\Cref{sec:general-geom}.

\begin{theorem}
\label{thm:local-MD-general-geom}
    Let $1 < p \le 2$.
    Assume $\diam_p(\mc{X}) \le \rad$ and 
    $f(x;z)$ is convex and $\lip$-Lipschitz with respect to 
    $\norm{\cdot}_p$ for all $z \in \domain$. Then there is an $(\diffp,\delta)$-DP algorithm that uses $O(\log n \cdot \min(n^{3/2} \sqrt{\log d}, n^2 \diffp/\sqrt{d})) $ and outputs $\hat x$ such that
    \begin{equation*}
       \E[F(\hat x) - F(x\opt)] 
       = \lip \rad  \cdot O \Bigg( \frac{1}{\sqrt{(p-1) n}}
            + \frac{ \sqrt{d \log d \log \tfrac{1}{\delta} }}{(p-1) n \diffp} \Bigg). 
    \end{equation*}
    If $p=2$ then the output $\hat x$ has
    \begin{equation*}
       \E[F(\hat x) - F(x\opt)] 
       = \lip \rad  \cdot O \Bigg( \frac{1}{\sqrt{ n}}
            + \frac{ \sqrt{d \log \tfrac{1}{\delta} }}{n \diffp} \Bigg). 
    \end{equation*}
\end{theorem}

\section{Efficient Algorithms for Smooth Functions}
\label{sec:FW}
Having established tight bounds for the non-smooth case,
in this section we turn to the smooth setting and develop linear-time private Frank-Wolfe algorithms with variance-reduction that achieve the optimal rates. Specifically, our algorithms achieve the rate $\wt{O}({1}/{\sqrt{n \diffp}})$ for pure $\diffp$-DP and $\wt{O} \left( {1}/{\sqrt{n}} + {1}/{(n \diffp)^{2/3}}  \right)$ for \ed-DP. These results imply that the optimal (non-private) statistical rate $\wt{O}({1}/{\sqrt{n}})$ is achievable with strong privacy guarantees---whenever $\diffp \ge \wt{\Omega}(1/n^{1/4})$ for \ed-DP---even for high dimensional functions with $d \gg n$.


The starting point of our algorithms is the recent non-private Frank-Wolfe algorithm of~\citet{YurtseverSrCe19} 
which uses variance-reduction techniques to achieve the (non-private) optimal rates. Due to the high reuse of samples, a direct approach to privatizing their algorithm would results in sub-optimal bounds. To overcome this, we design a new binary-tree scheme for variance reduction that allows for more noise-efficient private algorithms.

\newcommand{\bs}{b} 

We describe our private Frank-Wolfe procedure in~\cref{alg:FW-tree}. We present the algorithm in a more general setting where $\xdomain$ can be an arbitrary convex body with $m$ vertices. The algorithm has $T$ phases (outer iterations) indexed by $1 \le t \le T$ and each phase $t$ has a binary tree of depth $t$. We will denote vertices by $u_s$ where $s \in \{0,1\}^{\le t}$ is the path to the vertex; i.e., $u_\emptyset$ will denote the root of the tree, 
$u_{01}$ will denote the right child of $u_{0}$.
Each vertex $u_s$ is associated with a parameter $x_{t,s}$, a gradient estimate $v_{t,s}$, and a set of samples $S_{t,s}$ of size $2^{-j} \bs$ where $j$ is the depth of the vertex. Roughly, the idea is to improve the gradient estimate at a vertex (reduce the variance) using the gradient estimates at vertices along the path to the root; since these vertices have more samples, this procedure can result in better gradient estimates. 

More precisely, the algorithm traverses through the graph vertices according to the Depth-First-Search (DFS) approach. 
At each vertex, the algorithm improves the gradient estimate at the current vertex using the estimate at the parent vertex. When the algorithm visits a leaf vertex, it also updates the current iterate using the Frank-Wolfe step with the gradient estimate at the leaf.

For notational convenience, we let $\mathsf{DFS}(t)$ denote the DFS order of the vertices in a binary tree of depth $t$ (root not included), i.e., for $t=2$ we have $\mathsf{DFS}(t)= \{u_0,u_{00},u_{01},u_{1},u_{10},u_{11}\}$. Moreover, for $s \in \{0,1\}^t$ we let $\ell(s)$ denote the integer whose binary representation is $s$. In the description of the algorithm, we denote iterates by $x_{t,s}$ where $t$ is the phase and $s \in \{0,1\}^t$ is the path from the root. In our proofs, we sometimes use the equivalent notation $x_k$ where $k = 2^{t-1} + \ell(s)$.

\begin{figure*}
\begin{center}
\resizebox{0.75\textwidth}{!}{
\begin{tikzpicture} 
    \node [circle,draw=red, thick]{$u_\emptyset$} [level distance=10mm,sibling distance=62mm] 
    child { node [circle,draw]{$u_0$} [level distance=10mm ,sibling distance=30mm]  
    child {node [circle,draw] {$u_{00}$} [level distance=10mm ,sibling distance=15mm] 
    child {node [circle,draw] {$u_{000}$}} 
    child {node [circle,draw]{$u_{001}$}}} 
    child {node [circle,draw]{$u_{01}$} [level distance=10mm ,sibling distance=15mm] 
    child {node [circle,draw] {$u_{010}$}} 
    child {node [circle,draw]{$u_{011}$}}} } 
    child {node [circle,draw=red, thick] {$u_1$} [level distance=10mm ,sibling distance=30mm] 
    child {node [circle,draw=red, thick] {$u_{10}$} [level distance=10mm ,sibling distance=15mm]  
    child {node [circle,draw] {$u_{100}$}} 
    child {node [circle,draw=red, thick]{$u_{101}$}}} 
    child {node [circle,draw=green, thick]{$u_{11}$} [level distance=10mm ,sibling distance=15mm] 
    child {node [circle,draw] {$u_{110}$}} 
    child {node [circle,draw]{$u_{111}$}}} 
    }; 
\end{tikzpicture}
}
\end{center}
\caption{Binary tree at phase $t=3$ of the algorithm. At the leaf $u_{101}$, the algorithm has the gradient estimate $v_{t,101}$ which is calculated along the path to the root where every right son applies a correction step to the estimate. Using the gradient estimate $v_{t,101}$, the algorithm applies a Frank-Wolfe step to calculate the next iterate and put its value in the next DFS vertex, namely $u_{t,11}$.}
\label{fig:tree}
\end{figure*}
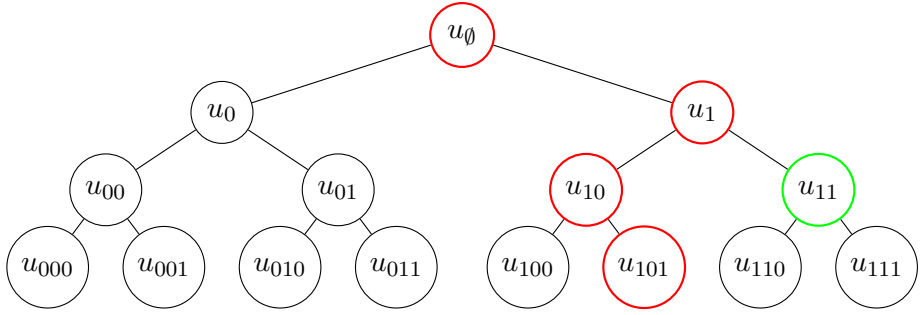

\begin{algorithm}[ht]
	\caption{Private Variance Reduced Frank-Wolfe}
	\label{alg:FW-tree}
	\begin{algorithmic}[1]
		\REQUIRE Dataset $\Ds=(\ds_1, \ldots, \ds_n)\in \domain^n$, 
		constraint set 
		$\xdomain = \mathsf{conv}\{c_1,\dots,c_m\}$,
		number of phases $T$,
		batch size $\bs$,
		initial point $x_0$;
        \FOR{$t=1$ to $T$\,}
            \STATE Set $x_{t,\emptyset} = x_{t-1,L_{t-1}}$
            \STATE Draw $\bs$ samples to the set $S_{t,\emptyset}$
            \STATE  $v_{t,\emptyset} \gets \nabla f(x_{t,\emptyset}; S_{t,\emptyset})$
            \FOR{$u_s \in \mathsf{DFS}[2^t]$\,}
                \STATE Let $s = s' c$ where $c \in \{0,1\}$ and $j = |s|$
                \IF{$c=0$}
                    \STATE $v_{t,s} \gets v_{t,s'}$; \ $x_{t,s} \gets x_{t,s'}$
                \ELSE
                    \STATE Draw $2^{-j} \bs  $ samples to the set $S_{t,s}$
                    \STATE $v_{t,s} \gets v_{t,s'} 
        	        + \nabla f(x_{t,s};S_{t,s})
                    - \nabla f(x_{t,s'};S_{t,s})$
                \ENDIF    
                \IF{$j=t$}
                    \STATE Let $s_+$ be the next vertex in the DFS iteration
                    \STATE $w_{t,s} \gets \argmin_{c_i : 1 \le i \le m} \<c_i, v_{t,s} \> + \noise_i$ where $\noise_i \sim \mathsf{Laplace}(\lambda_{t,s})$
        	    \STATE $x_{t,s_+} \gets (1 - \ss_{t,s}) x_{t,s} + \ss_{t,s} w_{t,s}$ where
        	$\ss_{t,s} = \frac{2}{2^{t-1} + \ell(s) + 1}$
                \ENDIF
        	     
            \ENDFOR
        
        \ENDFOR
            \RETURN the final iterate $x_K$
	\end{algorithmic}
\end{algorithm}

We analyze~\cref{alg:FW-tree} for pure and approximate DP.

\subsection{Pure Differential Privacy}
\label{sec:FW-pure}
The following theorem summarizes our guarantees 
for pure privacy. 
\iftoggle{arxiv}{}{We defer missing proofs to~\cref{sec:proofs-FW}.}
\begin{theorem}
\label{thm:FW-pure}
    Assume that $\diam_1(\xdomain) \le \rad$, $m \le O(d)$ and that $f(x;z)$ is convex,  $\lip$-Lipschitz and $\sm$-smooth
    with respect to $\lone{\cdot}$.
    Assume also that $\frac{\lip \log m \log^2 n}{n \diffp \rad} \le \sm \le \frac{n \lip \log m}{\diffp \rad \log^2 n}$.
    Setting  $\bs = n/\log^2 n $, $\lambda_{t,s} =\frac{2 \lip \rad 2^t }{\bs \diffp}$ and $T = \frac{1}{2} \log \left( \frac{\bs \diffp \sm \rad}{\lip \log m} \right)$,
    \cref{alg:FW-tree} is $\diffp$-DP, queries $n$ gradients, and has
    \begin{align*}
     \iftoggle{arxiv}{}{&} \E[F(x_{K}) - F(x\opt)] 
     \iftoggle{arxiv}{}{\\}
        & \iftoggle{arxiv}{}{\quad} 
        \le O \left( \rad (\lip + \sm \rad) \frac{\sqrt{\log d} \log n}{\sqrt{n}} 
        + \frac{\sqrt{\sm \lip \rad^3 \log d} \log n}{\sqrt{n \diffp}} \right).
    \end{align*}
    Moreover, if $\sm \le \frac{\lip \log m \log^2 n}{n \diffp \rad}$ then setting $T=1$ and $\bs = n$, \cref{alg:FW-tree} is $\diffp$-DP, queries $n$ gradients, and has
     \begin{equation*}
     \E[F(x_{K}) - F(x\opt)] 
        \le \rad \lip  \cdot  O\left(  \frac{\sqrt{\log d} }{\sqrt{n}} 
        +  \frac{\log d}{n \diffp}\right) + O( \sm \rad^2).
    \end{equation*}
\end{theorem}

To prove the theorem, we begin with the following lemma that gives pure privacy guarantees.
\begin{lemma}
\label{lemma:FW-pure-priv}
    Assume $2^T \le \bs$.
    Setting $\lambda_{t,s} = \frac{ 2 \lip \rad 2^t }{ \bs \eps}$, \cref{alg:FW-tree} is $\diffp$-DP with $\diffp \le 1$. Moreover, $\E[\<v_{t,s}, w_{t,s}\>] \le \E[\min_{w \in \xdomain} \<v_{t,s},w\>] + O(\frac{\lip \rad 2^t}{ \bs \eps} {\log m})$.
\end{lemma}
\iftoggle{arxiv}{
\begin{proof}
The main idea for the privacy proof is that 
each sample in the set $S_{t,s}$ is used in the calculation of $v_{t,s}$ at most $N_{t,s} = 2^{t - |s|}$ times,
hence setting the noise large enough so that each iterate is $\frac{\diffp}{N_{t,s}}$-DP, we get that the final mechanism is $\diffp$-DP using basic composition. Let us now provide a more formal argument.
Let $\Ds=(\ds_1,\dots,\ds_{n-1},\ds_n),\Ds'=(\ds_1,\dots,\ds_{n-1},\ds'_n)$ be two neighboring datasets with iterates $x = (x_1,\dots,x_K)$
and $x' = (x'_1,\dots,x'_K)$, respectively. We prove that $x$ and $x'$
are $\diffp$-indistinguishable, i.e., $x \approx_{(\eps,0)} x'$.
Let $S_{t,s}$ be the set (vertex) that contains the last sample (i.e., $z_n$ or $z'_n$) and let 
$j = |s|$ denote the depth of this vertex. 
We will prove privacy given that the $n$'th sample is in $S_{t,s}$, which will imply our general privacy guarantee as this holds for every choice of $t$ and $s$.
    
Note that $|S_{t,s}| = 2^{-j} \bs$ and that this set is used in the calculation of $v_k$ for at most $2^{t-j}$ (consecutive) iterates, namely these are leafs that are descendants of the vertex $u_{t,s}$.
Let $k_0$ and $k_1$ be the first and last iterate such that the set $S_{t,s}$ is used for the calculation of $v_k$, hence $k_1 - k_0 + 1 \le  2^{t - j} $.
The iterates $(x_1,\dots,x_{k_0-1})$ and $(x'_1,\dots,x'_{k_0-1})$ do not depend on the last sample and therefore has the same distribution (hence $0$-indistinguishable).
Moreover, given that $(x_{k_0},\dots,x_{k_1}) \approx_{(\diffp,0)} (x'_{k_0},\dots,x'_{k_1})$, it is clear that
the remaining iterates $(x_{k_1+1},\dots,x_{K}) \approx_{(\diffp,0)} (x'_{k_1+1},\dots,x'_{K})$
by post-processing as they do depend on the last sample only through the previous iterates. It is therefore enough to prove that 
$(x_{k_0},\dots,x_{k_1}) \approx_{(\diffp,0)} (x'_{k_0},\dots,x'_{k_1})$. 
To this end, we prove that for each such iterate, $w_k \approx_{({\eps}/{ 2^{t - j}},0)} w'_k$, hence using post-processing and basic composition the iterates are 
$\eps$-indistinguishable as $k_1 - k_0 + 1 \le 2^{t - j}$.
Note that for every $k_0 \le k \le k_1$ the sensitivity $|\<c_i, v_k - v'_k \> | \le  \frac{\rad \lip}{2^{-j} \bs }$.
Hence, using privacy guarantees of report noisy max~\citep[claim 3.9]{DworkRo14}, we have that $w_k \approx_{({\eps}/{ 2^{t - j}},0)} w'_k$ since 
$\lambda_{t,s} = \frac{ 2 \lip \rad 2^t }{ \bs \eps}$.
    
Now we prove the second part of the claim. 
Standard results for the expectation of the maximum of $m$ Laplace random variables imply that $\E[\<v_{t,s}, w_{t,s}\>] \le \min_{1 \le i \le m } \<v_{t,s},c_i\> + O(\frac{\lip \rad 2^t}{\bs \eps} {\log m})$.
Since $\xdomain = \mathsf{conv}\{c_1,\dots,c_m\}$, we know that for any $v \in \R^d$, $\argmin_{w \in \xdomain} \<w,v\>  \cap \{c_1,\dots,c_m\} \neq \emptyset $~\citep{TalwarThZh15}(Fact 2.3) which proves the claim.
\end{proof}
}{
\begin{proof}(sketch)
    Let $\Ds=(\ds_1,\dots,\ds_{n-1},\ds_n)$ and $\Ds'=(\ds_1,\dots,\ds_{n-1},\ds'_n)$ be two neighboring datasets.
    We prove privacy given that the $n$'th sample belongs to the set $S_{t,s}$, which will imply our general privacy guarantee as this holds for every choice of $t$ and $s$. The main idea is that each sample in the set $S_{t,s}$ is (directly) involved in the 
    calculation of a Frank-Wolfe update at most $N_{t,s} = 2^{t - |s|}$ times. Hence, setting the noise level $\lambda_{t,s}$ large enough to guarantee that each iterate is ${\diffp}/{N_{t,s}}$-DP, basic composition implies the final output is $\diffp$-DP.
\end{proof}
}

The next lemma upper bounds the variance of the gradients. 
\begin{lemma}
\label{lemma:FW-var-tree}
    At the vertex $(t,s)$, we have
    \begin{equation*}
        \E\linf{v_{t,s} - \nabla F(x_{t,s})} \le (\lip + \sm \rad) \cdot O \left( \sqrt{\log (d)/b} \right).
    \end{equation*}
\end{lemma}

\iftoggle{arxiv}{
The claim follows directly from the following lemma.
\begin{lemma}
\label{lemma:coord-subg-tree}
    Let $(s,t)$ be a vertex and $\sigma^2 =  (\lip^2  + \sm^2 \rad^2)/\bs$.
    For every index $1 \le i \le d$, 
    \begin{equation*}
    \E \left[ e^{\lambda (v_{t,s,i} - \nabla F_i(x_{t,s}))} \right] 
        \le e^{O(1) \lambda^2 \sigma^2}.
    \end{equation*}
\end{lemma}

\begin{proof}(\Cref{lemma:FW-var-tree})
\Cref{lemma:coord-subg-tree} says that $v_{t,s,i} - \nabla F_i(x_{t,s})$ is $O(\sigma^2)$-sub-Gaussian for every $1 \le i \le d$, hence standard results imply that the maximum of $d$ sub-Gaussian random variables is 
$\E\linf{v_{t,s} - \nabla F(x_{t,s})} \le O(\sigma) \sqrt{\log d}$. The claim follows.
\end{proof}

\begin{proof}(\Cref{lemma:coord-subg-tree})
    Let us fix $i$ for simplicity and let 
    $B_{t,s} = v_{t,s,i} - \nabla F_i(x_{t,s})$.
    We prove the claim by induction on the depth of the vertex, i.e., $j = |s|$.
    If $j=0$ then $s = \emptyset$ which implies that 
    $v_{t,\emptyset} = \nabla f(x_{t,\emptyset}; S_{t,\emptyset})$ where $S_{t,\emptyset}$ is a sample of size $b$. Thus we have
    \begin{align*}
    \E[e^{\lambda B_{t,\emptyset}}]
       & = \E \left[ e^{\lambda (v_{t,\emptyset,i} - \nabla F_i(x_{t,\emptyset})} \right] \\
       & = \E \left[ e^{\lambda (\frac{1}{\bs} \sum_{s \in S_{t,\emptyset}} \nabla f_i(x_{t,\emptyset};s) - \nabla F_i(x_{t,\emptyset})} \right] \\
       & = \prod_{s \in S_{t,\emptyset}}
            \E[ e^{\frac{\lambda}{\bs} (\nabla f_i(x_{t,\emptyset};s) - \nabla F_i(x_{t,\emptyset}))}  ] \\
        & \le  e^{\lambda^2 \lip^2/2 \bs},
    \end{align*}
    where the last inequality follows since for a random variable $X \in [-\lip,\lip]$ and $\E[X]=0$, we have $\E[e^{\lambda X}] \le e^{\lambda^2 \lip^2/2}$~(\citealp{Duchi19}, example 3.6).
    Assume now we have $s$ with $|s| = j >0$ and let $s = s' c$ where $c \in \{0,1\}$. If $c=0$ the claim clearly holds so we assume $c=1$.
    Recall that in this case
    $v_{t,s} = v_{t,s'} + \nabla f(x_{t,s};S_{t,s}) - \nabla f(x_{t,s'};S_{t,s})$, hence
    $B_{t,s} = v_{t,s,i} - \nabla F_i(x_{t,s})
    = B_{t,s'} + \nabla f_i(x_{t,s};S_{t,s}) - \nabla f_i(x_{t,s'};S_{t,s}) - \nabla F_i(x_{t,s}) + \nabla F_i(x_{t,s'})$
    Letting $S_{< {t,s}} = \cup_{(t_1,s_1) < (t,s)} S_{t_1,s_1}$ be the set of all samples used up to vertex $t,s$, the law of iterated expectation implies
    \begin{align*}
    \E[e^{\lambda B_{t,s}}] 
        & = \E[e^{\lambda( B_{t,s'} + \nabla f_i(x_{t,s};S_{t,s}) - \nabla f_i(x_{t,s'};S_{t,s}) - \nabla F_i(x_{t,s}) + \nabla F_i(x_{t,s'}) )}] \\
        & = \E \left[ \E[e^{\lambda( B_{t,s'} + \nabla f_i(x_{t,s};S_{t,s}) - \nabla f_i(x_{t,s'};S_{t,s}) - \nabla F_i(x_{t,s}) + \nabla F_i(x_{t,s'}) )}] \mid S_{<(t,s)}\right] \\
        & = \E \left[ \E[e^{\lambda B_{t,s'} } \mid S_{<(t,s)} ] \cdot
        \E[e^{\lambda( \nabla f_i(x_{t,s};S_{t,s}) - \nabla f_i(x_{t,s'};S_{t,s}) - \nabla F_i(x_{t,s}) + \nabla F_i(x_{t,s'}) )} \mid S_{<{t,s}}  ] \right] \\
        & =  \E[e^{\lambda B_{t,s'} }] \cdot
        \E[e^{\lambda( \nabla f_i(x_{t,s};S_{t,s}) - \nabla f_i(x_{t,s'};S_{t,s}) - \nabla F_i(x_{t,s}) + \nabla F_i(x_{t,s'}) )} \mid S_{<{t,s}}].
    \end{align*}
    Since $f(\cdot;s)$ is $\sm$-smooth with respect to $\lone{\cdot}$, we have that $|\nabla f_i(x_{t,s};S_{t,s}) - \nabla f_i(x_{t,s'};S_{t,s})| \le \sm \lone{x_{t,s} - x_{t,s'}}$.
    Moreover, as $u_{t,s}$ is the right son of $u_{t,s'}$,
    the number of updates between $x_{t,s}$ and $x_{t,s'}$
    is at most the number of leafs visited between these two vertices
    which is $2^{t-j}$. Hence we get that
    \begin{equation*}
    \lone{x_{t,s} - x_{t,s'}} 
       \le \rad \ss_{t,s'} 2^{t-j}  
        \le \rad 2^{-j+2},    
    \end{equation*}
    which implies that  $|\nabla f_i(x_{t,s};S_{t,s}) - \nabla f_i(x_{t,s'};S_{t,s})| \le \sm \rad 2^{-j+2}$. 
    Since $\E[\nabla f_i(x_{t,s};S_{t,s}) - \nabla f_i(x_{{t,s'}};S_{t,s})  \mid S_{< {t,s}} ] = \nabla F_i(x_{t,s}) - \nabla F_i(x_{t,s'})$,
    by repeating similar arguments to the case $\ell=0$, we get that
    \begin{align*}
     \E[e^{\lambda( \nabla f_i(x_{t,s};S_{t,s}) - \nabla f_i(x_{t,s'};S_{t,s}) - \nabla F_i(x_{t,s}) + \nabla F_i(x_{t,s'}) )} \mid S_{< {t,s}}]
     & \le e^{O(1) \lambda^2 \sm^2 \rad^2 2^{-2j}/|S_{t,s}|} \\
     & \le e^{O(1) \lambda^2 \sm^2 \rad^2 2^{- j}/ \bs }.
    \end{align*}
    Overall we have that $\E[e^{\lambda B_{t,s}}] \le \E[e^{\lambda B_{t,s'}}] \cdot e^{O(1) \lambda^2 \sm^2 \rad^2 2^{-j}/\bs }$. Applying this inductively, we get that for every $(t,s)$
    \begin{equation*}
    \E[e^{\lambda B_{t,s}}]
        \le e^{O(1) \lambda^2 (\lip^2 + \sm^2 \rad^2)/ \bs}.
        \qedhere
    \end{equation*}
\end{proof}
}{}

Using the previous two lemmas, we can prove~\cref{thm:FW-pure}.
\begin{proof}
The setting of the parameters and the condition on $\sm$ ensures that $2^T \le \bs$ hence~\cref{lemma:FW-pure-priv} implies the claim about privacy. Now we proceed to prove utility.
In this proof, we use the equivalent representation $k = 2^{t-1} + \ell(s)$
for a leaf vertex $(t,s)$ where $\ell(s)$ is the number whose binary
representation is $s$.
By smoothness we get,
\begin{align*}
   \iftoggle{arxiv}{}{&} F(x_{k+1}) 
   \iftoggle{arxiv}{}{\\} 
        & \le F(x_k) + \< \nabla F(x_k), x_{k+1} - x_k \> + \sm \lone{x_{k+1} - x_k}^2 / 2 \\
        & \le F(x_k) + \ss_k \< \nabla F(x_k), w_k - x_k \> + {\sm \ss_k^2 \rad^2 }/{2} \\
        & = F(x_k) 
        + \ss_k \<\nabla F(x_k), x\opt - x_k \>
        + \ss_k \< v_k, w_k - x\opt \>  \\
        & \quad + \ss_k \< \nabla F(x_k) - v_k, w_k - x\opt \> + {\sm \ss_k^2 \rad^2 }/{2} \\
        & \le F(x_k) 
        + \ss_k (F(x\opt) - F(x_k)) 
        + \ss_k \rad \linf{\nabla F(x_k) - v_k} 
        \\
        & \quad + \ss_k (\< v_k, w_k\>  -  \min_{w \in \xdomain} \<v_k,w \>) 
        + {\sm \ss_k^2 \rad^2 }/{2}
        .
\end{align*}
Subtracting $F(x\opt)$ from each side,
using~\cref{lemma:FW-pure-priv,lemma:FW-var-tree} and taking expectations, we have
\begin{align*}
    \iftoggle{arxiv}{}{&} \E[F(x_{k+1}) - F(x\opt)] \iftoggle{arxiv}{}{\\}
    & \le  (1- \ss_k) \E[F(x_k) - F(x\opt)]
        + \ss_k \rad (\lip + \sm \rad) \sqrt{\frac{\log d}{\bs}} \\
        & \quad + \frac{\ss_k^2}{2} \sm \rad^2
        + \ss_k \rad \lip  \frac{2^t \log m}{\bs \diffp}.
\end{align*}
Letting $\alpha_k = \ss_k \rad (\lip + \sm \rad) \sqrt{\frac{\log d}{\bs}}
        + \frac{\ss_k^2}{2} \sm \rad^2
        + \ss_k \rad \lip  \frac{2^t \log m}{\bs \diffp}$,
we have
\begin{align*}
    \E[F(x_{K}) - F(x\opt)] 
        & \le \sum_{k=1}^K \alpha_k \prod_{i >k} (1 - \ss_i) \\
        & = \sum_{k=1}^K  \alpha_k \frac{(k-1) k}{K(K+1)} 
        \le \sum_{k=1}^K  \alpha_k \frac{k^2}{K^2}.
\end{align*}
%
%
Since $t \le T$ and $K = 2^T$, simple algebra now yields
\begin{align*}
    \iftoggle{arxiv}{}{&} \E[F(x_{K}) - F(x\opt)] 
    \iftoggle{arxiv}{}{\\}
        & \le O \left( \rad (\lip + \sm \rad) \frac{\sqrt{\log d}}{\sqrt{\bs}} 
        + \frac{\sm \rad^2}{2^T} 
        + \rad \lip  \frac{2^T \log m}{\bs \diffp} \right).
\end{align*}
The number of samples in the algorithm is upper bounded by $T^2 \cdot \bs $ hence the first part of the claim follows by
setting $\bs = n/\log^2 n$ and $T = \frac{1}{2} \log \left( \frac{\bs \diffp \sm \rad}{\lip \log m} \right)$.
The condition on $\sm$ ensures that the term inside the log is greater than $1$. The second part follows similarly using $T=1$ and $\bs = n$.
\end{proof}

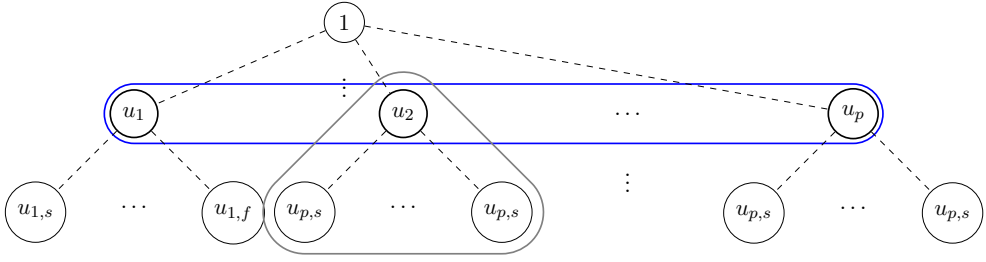
\begin{figure*}
\begin{center}
\resizebox{0.80\textwidth}{!}{
\begin{tikzpicture}
\node(root)[circle,draw]{$1$};

\node (vdotsr) [below=.25cm of root] {$\vdots$};

\node (0)  [circle,draw, below left = 1cm and 3cm of root,thick] {$u_1$};
\node (1) [circle,draw, right=3.7cm of 0, thick] {$u_2$};
\node (ldots) [right=3cm of 1] {$\ldots$};
\node (2) [circle,draw, right=3 cm of ldots,thick] {$u_p$};

\draw[thick,blue] \convexpath{0,1,2}{0.5cm};

              

\draw[dashed] (root)--(0);
\draw[dashed] (root)--(1);
\draw[dashed] (root)--(2);
\node (vdots) [below=.5cm of ldots] {$\vdots$};

\node (00)  [circle,draw, below left = 1cm and 1cm of 0] {$u_{1,s}$};
\node (ldots0) [below = 1cm of 0] {$\ldots$};
\node (01) [circle,draw, below right = 1cm and 1cm of 0] {$u_{1,f}$};
\draw[dashed] (0)--(00);
\draw[dashed] (0)--(01);

\node (10)  [circle,draw, below left = 1cm and 1cm of 1] {$u_{p,s}$};
\node (ldots1) [below = 1cm of 1] {$\ldots$};
\node (11) [circle,draw, below right = 1cm and 1cm of 1] {$u_{p,s}$};
\draw[dashed] (1)--(10);
\draw[dashed] (1)--(11);
\draw[thick,gray] \convexpath{10,1,11}{0.7cm};

\node (20)  [circle,draw, below left = 1cm and 1cm of 2] {$u_{p,s}$};
\node (ldots2) [below = 1cm of 2] {$\ldots$};
\node (21) [circle,draw, below right = 1cm and 1cm of 2] {$u_{p,s}$};
\draw[dashed] (2)--(20);
\draw[dashed] (2)--(21);

\end{tikzpicture}
}
\end{center}
\caption{
We view~\cref{alg:FW-tree} 
as a sequence of algorithms $\A_1,\dots,\A_p$, each $\A_i$ operating on the subtree of the vertex $u_i$ using the outputs of the previous algorithms. 
The gray set denotes the subtree over which $\A_2$ operates; its outputs 
are the iterates corresponding to the leafs of this subtree.
If each $\A_i$ is $(\diffp_0,\delta_0)$-DP,
then shuffling the samples at nodes of depth $j$ (in blue) amplifies the privacy to roughly $({\diffp_0 \sqrt{\log(1/\delta)/2^j}},\delta + n \delta_0 )$-DP.}
\label{fig:tree-shuffle}
\end{figure*}

\subsection{Approximate Differential Privacy}
\label{sec:FW-appr}
The previous section achieves the optimal non-private rate ${1}/{\sqrt{n}}$
only for $\diffp = \Theta(1)$. In this section we show that for approximate differential privacy, it is possible to achieve the optimal rates when $\diffp \ge \Omega(n^{-1/4})$. The first approach to improve the privacy analysis here is to use advanced composition for approximate DP. Unfortunately, it is not enough by itself and we use  
amplification by shuffling results 
to achieve the optimal bounds. The following theorem summarizes the guarantees of~\cref{alg:FW-tree} for approximate privacy.
\begin{theorem}
\label{thm:FW-appr}
    Let $\delta \le 1/n$ and 
    assume that $\diam_1(\xdomain) \le \rad$, $m \le O(d)$ and that $f(x;z)$ is convex,  $\lip$-Lipschitz and $\sm$-smooth
    with respect to $\lone{\cdot}$.
    Assume $\frac{\lip \log(n/\delta) \log m \log^2 n}{n \diffp \rad} \le \sm \le \frac{\sqrt{n} \lip \log(n/\delta) \log m}{\diffp \rad \log n}$ and 
    $\diffp \le \frac{(\lip \log(n/\delta) \log m)^{1/4} \sqrt{\log n}}{(n \sm \rad)^{1/4}}$.
    Let  $\lambda_{t,s} = \frac{ \lip \rad  2^{T/2} \log(n/\delta)}{ \bs \diffp}$, $\bs = n /\log^2 n$, and $T = \frac{2}{3} \log \left( \frac{\bs \diffp \sm \rad}{ \lip \log(n/\delta) \log m} \right)$, then
    \cref{alg:FW-tree} is $(\diffp,\delta)$-DP, queries $n$ gradients, 
    and has
    \begin{align*}
     \E[F(x_{K}) \!-\! F(x\opt)] 
        & \le O \left( \rad (\lip + \sm \rad) \frac{\sqrt{\log d} \log n}{\sqrt{n}} \right) \iftoggle{arxiv}{}{\\
        &} + O\left( \frac{\sqrt{\sm} \lip \rad^{2} \log (1/\delta) \log m
        \log^2 n }{n\diffp} \right)^{2/3} 
        \!\!\!\!.
    \end{align*}
\end{theorem}


The following lemma proves privacy in this setting.
\begin{lemma}
\label{lemma:FW-appr-priv}
    Let $2^T \le \bs$, $\delta \le 1/n$ and $\diffp \le \sqrt{2^{-T} \log(1/\delta)}$.
    Setting $\lambda_{t,s} = \frac{ \lip \rad  2^{T/2} \log(n/\delta)}{ \bs \diffp}$, \cref{alg:FW-tree} is $(O(\diffp) ,21 \delta)$-DP.
    Moreover, $\E[\<v_{t,s}, w_{t,s}\>] \le \E[\min_{w \in \xdomain} \<v_{t,s},w\>] + O({\lip \rad {2^{T/2}\log(n/\delta)}} {\log(m)}/{ \bs \diffp})$.
\end{lemma}
\iftoggle{arxiv}{
To prove~\cref{lemma:FW-appr-priv}, we use the following privacy amplification by shuffling.
\begin{lemma}[\citealp{FeldmanMcTa20}, Theorem 3.8] 
\label{lemma:amp-shuff}
    Let $\A_i: \range^{i-1} \times \domain \to \range$ for $i \in [n]$ be a sequence of algorithm such that $\A_i(w_{1:i-1},\cdot)$ is $(\diffp_0,\delta_0)$-DP for all values of $w_{1:i-1} \in \range^{i-1}$ with $\diffp_0 \le O(1)$. Let $\A_S : \domain^n \to \range^n$ be an algorithm that given $z_{1:n} \in \domain^n$, first samples a random permutation $\pi$, then sequentially computes $w_i = \A_i(w_{1:i-1},z_{\pi(i)})$ for $i \in [n]$ and outputs $w_{1:n}$. Then for any $\delta$ such that  $\diffp_0 \le \log(\frac{n}{16 \log(2/\delta)})$, the algorithm $\A_s$ is $(\diffp, \delta + 20 n \delta_0)$
    where $\diffp \le O ( {\diffp_0 \sqrt{\log(1/\delta)/n}})$.
\end{lemma}

\begin{proof}
    
We use the same notation as~\cref{lemma:FW-pure-priv} where 
$\Ds=(\ds_1,\dots,\ds_{n-1},\ds_n),\Ds'=(\ds_1,\dots,\ds_{n-1},\ds'_n)$ denote two neighboring datasets with iterates $x = (x_1,\dots,x_K)$
and $x' = (x'_1,\dots,x'_K)$. Here, we prove privacy after conditioning on the event that the $n$'th sample is sampled at phase $t$ and depth $j$.
We need to show that the iterates are $(\diffp,\delta)$-indistinguishable.
We only need to prove privacy for the iterates at phase $t$ as the iterates before phase $t$ do not depend on the $n$'th sample and the iterates after phase $t$ are $(\diffp,\delta)$-indistinguishable by post-processing. 

Let us now focus on the iterates at phase $t$. Let $u_1,\dots,u_p$ denote the vertices at level $j$ that has samples $S_1,\dots,S_p$ each of size
$|S_i| = 2^{-j} \bs$. We will have two steps in the proof. First, we use advanced composition to show that the iterates that are descendant of a vertex $u_i$ are $(\diffp_0,\delta_0)$-DP where roughly $\diffp_0 = {2^{j/2} \diffp}$. As we have $p = 2^j$ vertices at depth $j$, we then use the amplification by shuffling result to argue that the final privacy guarantee is 
$(\diffp,\delta)$-DP (see Fig.~\ref{fig:tree-shuffle} for a demonstration of the shuffling in our algorithm).

Let $\A_i$ be the algorithm that outputs the iterates corresponding to the leafs that are descendants of $u_i$; we denote this output by $O_i$. Note that the inputs of $\A_i$ are the samples at $u_i$, which we denote as $S_i$, and the previous outputs $O_1,\dots,O_{i-1}$. In this notation, we have that $O_i = \A_i(O_1,\dots,O_{i-1},S_i)$. 
We let $\A_i$, $S_i$ and $O_i$ denote the above quantities when the input dataset is $\Ds_i$ and similarly $\A'_i$, $S'_i$ and $O'_i$ for $\Ds'$.
To prove privacy, we need to show that $(O_1,\dots,O_p) \approx_{(\diffp,\delta)} (O'_1,\dots,O'_p)$, that is $(O_1,\dots,O_p)$ and $(O'_1,\dots,O'_p)$ are $(\diffp,\delta)$-indistinguishable

To this end, we first describe an equivalent sampling procedure for the sets $S_1,\dots,S_p$. Given $r$ samples, the algorithm basically constructs the sets $S_1,\dots,S_p$ by sampling uniformly at random $p$ sets of size $r/p$ without repetition. Instead, we consider the following sampling procedure. First, we randomly choose a set of size $p(r-1)$ samples that does not include the $n$'th sample and using this set we randomly choose $r/p - 1$ samples for each set $S_i$. Then, we shuffle the remaining $p$ samples and add each sample to the corresponding set. It is clear that this sampling procedure is equivalent.
We prove privacy conditional on the output of the first stage of the randomization procedure which will imply privacy unconditionally.

Assuming without loss of generality that the samples which remained in the second stage are $z_{n-p+1},\dots,z_n$, and letting $\pi: [p] \to \{n-p+1,\dots,n\}$ denote the random permutation of the second stage, the algorithms $\A_i$ and $\A'_i$ can be written as a function of the previous outputs and the sample $z_{\pi(i)}$. This is true since the $\Ds$ and $\Ds'$ differ in one sample and therefore the first $r/p-1$ samples in the sets $S_i$ and $S'_i$ are identical. Thus, we can write $O_i = \A_i(O_1,\dots,O_{i-1},z_{\pi(i)})$.

Using the above notation, we are now ready to prove privacy. First, we show privacy for each $i$ using advanced composition. Similarly to~\cref{lemma:FW-pure-priv}, as each iterate $k$ which is a leaf of $u_i$ has sensitivity  $|\<c_i, v_k - v'_k \> | \le  \frac{\rad \lip}{2^{-j} \bs }$, we have that $x_k$ and $x_k'$ are 
$\frac{\diffp}{2^{T/2-j} {\log(n/\delta)}}$-indistinguishable since 
$\lambda_{t,s} = \frac{ \lip \rad  2^{T/2} \log(n/\delta)}{ \bs \diffp}$. Since there are $2^{t-j}$ leafs of $u_i$, advanced composition (\cref{lemma:advanced-comp}) implies that $O_i \approx_{(\diffp_0, \delta_0)} O'_i$ 
where $\diffp_0 = \frac{\diffp}{2^{T/2-j} {\log(n/\delta)}} \sqrt{2^{t-j} \log(1/\delta_0)} 
\le \frac{O(\diffp)}{\sqrt{\log(1/\delta)} 2^{-j/2}} $ by setting $\delta_0 = \delta/n$.

Finally, we can use the amplification by shuffling result to finish the proof.
First, note that we need $\diffp_0 \le \log(\frac{2^j}{16 \log(2/\delta)})$
to be able to use~\cref{lemma:amp-shuff}.
If $2^j \le O(\log(1/\delta))$ then we do not need the amplification by shuffling result as $\diffp_0 \le O( \diffp 2^{j/2}/\sqrt{\log(1/\delta)}) \le O(\diffp)$.
Otherwise $2^j$ is large enough so that we can use~\cref{lemma:amp-shuff}.
Since each $\A_i$ and $\A'_i$ are $(\diffp_0,\delta_0)$-DP and since the second stage shuffles the inputs to each algorithm, \cref{lemma:amp-shuff} now implies that the outputs of the algorithms $\A_i$ and $\A'_i$ are
$(\diffp_f,\delta + 20 n \delta_0)$-DP where $\diffp_f \le \frac{\diffp_0 \sqrt{\log(1/\delta)}}{2^{j/2}} \le O(\diffp)$ which proves the claim.
\end{proof}
}{
\begin{proof}[Proof~(sketch)]
Let $\Ds=(\ds_1,\dots,\ds_{n-1},\ds_n)$ and $\Ds'=(\ds_1,\dots,\ds_{n-1},\ds'_n)$ denote two neighboring datasets with iterates $x = (x_1,\dots,x_K)$
and $x' = (x'_1,\dots,x'_K)$. Here, we prove privacy after conditioning on the event that the $n$'th sample is sampled at phase $t$ and depth $j$.
We only need to prove privacy for the iterates at phase $t$ as the iterates before phase $t$ do not depend on the $n$'th sample and the iterates after phase $t$ are $(\diffp,\delta)$-DP by post-processing. 

Let us focus on the iterates at phase $t$. Let $u_1,\dots,u_p$ denote the vertices at level $j$ that has samples $S_1,\dots,S_p$ each of size
$|S_i| = 2^{-j} \bs$. Let $\A_i$ denote the algorithm that outputs the iterates corresponding to the descendant of the vertex $u_i$.
The proof has two steps: first, we use advanced composition to show that each $\A_i$ is $(\diffp_0,\delta_0)$-DP where roughly $\diffp_0 = {2^{j/2} \diffp}$. Then, as we have $p = 2^j$ vertices at depth $j$ with random samples (that is, shuffled between vertices), we use the amplification by shuffling result~\cite{FeldmanMcTa20} (see~\cref{lemma:amp-shuff}) to argue that the final algorithm is 
$(\diffp,\delta)$-DP (see Fig.~\ref{fig:tree-shuffle} for a demonstration of the shuffling). 
\end{proof}
}
\iftoggle{arxiv}{
\Cref{thm:FW-appr} now follows using similar arguments to the proof of~\cref{thm:FW-pure}.
\begin{proof}
The assumptions on $\sm$ ensure that $2^T \le \bs$ and the assumptions on $\diffp$ ensure $\diffp \le 2^{-T/2} \log(n/\delta)$ hence the 
privacy follows from~\cref{lemma:FW-appr-priv}.
The utility analysis is similar to the proof of~\cref{thm:FW-pure}.
Repeating the same arguments in the proof of~\cref{thm:FW-pure} while using the new value of $\lambda_{t,s}$, we get
\begin{align*}
    \E[F(x_{K}) - F(x\opt)] 
        \le O \left( \rad (\lip + \sm \rad) \frac{\sqrt{\log d}}{\sqrt{\bs}} 
        + \frac{\sm \rad^2}{2^T} 
        + \rad \lip  \frac{2^{T/2} \log(n/\delta) \log m}{\bs \diffp} \right).
\end{align*}
As the number of samples is upper bounded by $T^2 \cdot \bs$, we set
 $T = \frac{2}{3} \log \left( \frac{\bs \diffp \sm \rad}{ \lip \log(n/\delta) \log m} \right)$
and $\bs = n / \log^2 n$ to get the first part of the theorem. Note that the condition on $\sm$ ensure the term inside the log is greater than $1$.
\end{proof}
}
{
\Cref{thm:FW-appr} now follows using similar arguments to the proof of~\cref{thm:FW-pure} (see~\cref{sec:thm-FW-appr}).
}

\section{Implications for Strongly Convex Functions}
\label{sec:strongly-convex}
When the function is strongly convex, we use standard reductions to the convex case to achieve better rates~\cite{FeldmanKoTa20}. Given a private algorithm $\A$ for the convex case, we use the following algorithm for the strongly convex case (see~\cite{FeldmanKoTa20}): run $\A$ for $k= \ceil{\log \log n}$ iterates, each initialized at the output of the previous iterates and run for $n_i = 2^{i-2} n / \log n$. Using this reduction with our algorithms for convex functions, we have the following theorems for non-smooth and smooth functions.

\begin{theorem}
\label{thm:sc-non-smooth}
    Assume $\diam_1(\mc{X}) \le \rad$ and 
    $f(x;z)$ is convex, $\lip$-Lipschitz, and $\sc$-strongly convex with respect to 
    $\lone{\cdot}$ for all $z \in \domain$.
    Then using~\cref{alg:phased-erm-MD} in the above algorithm  results in an algorithm that uses $O(\log n \log \log n\cdot \min(n^{3/2} \sqrt{\log d}, n^2 \diffp/\sqrt{d})) $ gradients and outputs $\hat x$ such that
    \begin{equation*}
       \E[F(\hat x) - F(x\opt)] 
       = \lip \rad  \cdot O \Bigg( \frac{{\log d}}{{n}}
            + \frac{ {d \log^3 d \log \tfrac{1}{\delta} }}{n^2 \diffp^2} \Bigg). 
    \end{equation*}
\end{theorem}

\begin{theorem}
\label{thm:sc-smooth}
    Let $\delta \le 1/n$ and 
    assume that $\diam_1(\xdomain) \le \rad$, $m \le O(d)$ and that $f(x;z)$ is convex,  $\lip$-Lipschitz, $\sc$-strongly convex and $\sm$-smooth
    with respect to $\lone{\cdot}$ where $\sm = O(\lip/\rad)$.
    Then using~\cref{alg:FW-tree} in the above algorithm results in an algorithm that uses $O(n) $ gradients and outputs $\hat x$ such that
    \begin{align*}
     \E[F(\hat x) \!-\! F(x\opt)] 
        & \le \lip \rad \cdot O \left( \frac{\log d \log^2 n}{n} \right) \iftoggle{arxiv}{}{\\
        &} +  \lip \rad \cdot O\left( \frac{ \log (1/\delta) \log m
        \log^2 n }{n\diffp} \right)^{4/3} 
        \!\!\!\!.
    \end{align*}
\end{theorem}
The proof follows directly from the proof of Theorem 5.1 in~\cite{FeldmanKoTa20}, together with the bounds of Section~\ref{sec:non-smooth} and Section~\ref{sec:FW}.

\section{Lower Bounds}
\label{sec:LB}
We conclude the paper with tight lower bounds. Our lower bounds are for the excess empirical loss but these can be translated to lower bounds for excess population loss using a simple bootstrapping 
\iftoggle{arxiv}{approach~\cite{BassilyFeTaTh19}.}{approach~(\citealp{BassilyFeTaTh19}).}

\subsection{Lower Bounds for Non-Smooth Functions}
\label{sec:LB-non-smooth}
In this section, we prove tight lower bounds for non-smooth functions using bounds for estimating the sign of the mean. In this problem, given a dataset $\Ds = (\ds_1,\dots,\ds_n)$ with mean $\bar z$, we aim to design private algorithms that estimate $\sign(\bar z)$. The following lemma provides a lower bound for this problem. We defer the proof to~\cref{sec:proof-lemma-LB-med}.
\begin{lemma}
\label{lemma:LB-med}
    Let $\Ds = (\ds_1,\dots,\ds_n)$ where 
    $\ds_i \in \domain = \{-\rad/d, \rad/d \}^d$ and
    let $\bar z = \frac{1}{n} \sum_{i=1}^n z_i$.
    Then any \ed-DP algorithm $\A: \domain \to \{-1,+1\}^d$ has
    \begin{equation*}
        \max_{\Ds \in \domain^n} \E \left[\sum_{j=1}^d |\bar{z}_j| 
		\indic {\A(\Ds)_j \neq \sign(\bar{z}_j)} \right] 
        \ge  \Omega \left(\frac{\rad \sqrt{d}}{n \diffp \log d} \right).
    \end{equation*}
\end{lemma}

The previous lemma implies our desired lower bound. 
\begin{theorem}
\label{thm:LB-non-smooth}
    Let $f(x;\ds_i) = \lip \lone{x-\ds_i}$ where $\ds_i\in \domain = \{-\rad/d, \rad/d \}^d$, $\hat F(x;\Ds) = \frac{1}{n} \sum_{i=1}^n \lone{x-\ds_i}$, and $\xdomain = \{x : \lone{x} \le \rad \}$.
    Then any \ed-DP algorithm $\A$ has
     \begin{equation*}
        \max_{\Ds \in \domain^n} \E \left[\hat F(\A(\Ds);\Ds) - \min_{x \in \xdomain} \hat F(x;\Ds)   \right] 
        \ge \Omega \left(\frac{\lip \rad \sqrt{d}}{n \diffp \log d} \right).
    \end{equation*}
\end{theorem}
\begin{proof}
    First, note that $f(x;\ds_i)$ is $\lip$-Lipschitz with respect to $\lone{\cdot}$. Moreover, it is immediate to see that the minimizer of $\hat F(\cdot;\Ds)$ is $x\opt = \sign(\bar z) \rad/d$ where $\bar z = \frac{1}{n} \sum_{i=1}^n z_i$ is the mean. 
    Letting $\hat x = \A(\Ds)$, simple algebra yields
    \begin{align*}
     \hat F(\hat x;\Ds) - \hat F(x\opt;\Ds) 
        \ge \lip \sum_{j=1}^d |\bar{z}_j| 
		\indic {\sign(\hat x_j) \neq \sign(\bar{z}_j)}.
    \end{align*}
    The claim now follows 
    from~\cref{lemma:LB-med} as $\sign(\A(\Ds))$ is differentially private by post-processing.
\end{proof}

\subsection{Lower Bounds for Smooth Functions}
\label{sec:LB-smooth}
In this section we prove tight lower bounds for smooth function. Specifically, we focus on $\sm$-smooth functions with $\sm \approx \lip / \rad$; such an assumption holds for many applications including LASSO (linear regression).
Our results in this section build on the lower bounds of~\citet{TalwarThZh15} which show tight bounds for private Lasso for sufficiently large dimension. 
We have the following lower bound for smooth functions which we prove in~\cref{sec:LB-smooth-app}.
%
\begin{theorem}
\label{thm:LB-smooth}
    Let $\xdomain = \{x \in \R^d : \lone{x} \le \rad \}$. There is family of convex functions $f: \xdomain \times \domain \to \R$ that is $\lip$-Lipschitz and $\sm$-smooth  with $\sm \le \lip/\rad$ such that any $(\diffp,\delta)$-DP algorithm $\A$ with $\delta = n^{-\omega(1)}$ has
    \begin{align*}
        \iftoggle{arxiv}{}{&} \sup_{\Ds \in \domain^n}
        \E \left[ \hat F(\A(\Ds);\Ds) - \min_{x \in \xdomain} \hat F(x;\Ds) \right] 
        \iftoggle{arxiv}{}{\\}
        & \iftoggle{arxiv}{}{\quad} 
        \ge \lip \rad \cdot \wt \Omega \left( \min \left( \frac{1}{(n \diffp)^{2/3}} ,\frac{\sqrt{d}}{n \diffp} \right) \right).
    \end{align*}
\end{theorem}
The lower bound of~\cref{thm:LB-smooth} implies the optimality of our upper bounds; if $d \ge \wt O((n\diffp)^{2/3})$ then the lower bound is essentially $1/(n\diffp)^{2/3}$ which is achieved by the private Frank-Wolfe algorithm of Section~\ref{sec:FW}, otherwise $d \le  \wt O((n\diffp)^{2/3})$ and the lower bound is ${\sqrt{d}}/{n \diffp}$ which is the same bound that private mirror descent (Section~\ref{sec:non-smooth}) obtains.





\printbibliography
\appendix

\newcommand{\MD}{\mathsf{MD}}
\section{Non-Contractivity of Mirror Descent}
\label{sec:MD-non-contractive}
In this section, we provide counter examples that show that Mirror Descent is not a contraction in general. To this end, we consider the standard mirror descent algorithm with KL-regularization over the simplex $\simplex_{d-1} = \{ x \in \R_+^d : \lone{x} = 1 \}$, that is, the following update with $h(x) = \sum_{j=1}^d x_j \log x_j$,
\begin{equation*}
    x^{k+1} = \argmin_{x \in \simplex} \left\{ \< \nabla f(x_k), x\> + \frac{1}{\ss_k} \db(x,x_k) \right\},
\end{equation*}
which yields the update
\begin{equation}
\label{eqn:MD-KL}
    x^{k+1} = \frac{x^{k} \cdot e^{-\ss \nabla f(x^k)}}
               {\lone{x^{k} \cdot e^{-\ss \nabla f(x^k)}}}.
\end{equation}
We let $x_{k+1} = \MD_{\ss,f}(x_k)$ denote the above mirror descent update.
The following lemma shows that mirror descent is not contractive even for linear functions.
\begin{lemma}
\label{lemma:MD-non-cont}
    There exists a linear function $f: \simplex_2 \to \R$ such that for every $0 < \ss \le 1$, there are $x_0,y_0 \in \simplex_2$ such that the mirror descent update $x_1 = \MD_{\ss,f}(x_0)$ and $y_1 = \MD_{\ss,f}(y_0)$ have
    \begin{equation*}
        \lone{x_1 - y_1} \ge (1 + \ss/4) \lone{x_0 - y_0},
        \quad 
        \db(x_1,y_1) \ge (1 + \ss/4) \db(x_0,y_0).
    \end{equation*}
\end{lemma}

\begin{proof}
We consider a linear function $f(x_1,x_2,x_3) = - x_2 - x_3$,
and two starting iterates for $n>0$ to be chosen presently
\begin{align*}
    x_0  = \left(1 - \frac{3}{n},\frac{1}{n},\frac{2}{n} \right),
    \quad
    y_0  = \left(1 - \frac{3}{n},\frac{2}{n},\frac{1}{n} \right).
\end{align*}
First, notice that for this setting of parameters, we
have that:
\begin{align*}
    \lone{x_0 - y_0} = \frac{2}{n},
    \quad
    \db(x_0,y_0) = \dkl(x_0,y_0) = \frac{\log 2}{n}.
\end{align*}
Using mirror descent update~\eqref{eqn:MD-KL}, we have 
\begin{align*}
    x_1  = \frac{1}{c} \left(x_{0,1},x_{0,2} ~ e^\ss,x_{0,3} ~ e^\ss \right),
    \quad
    y_1  = \frac{1}{c} \left(y_{0,1},y_{0,2} ~ e^\ss,y_{0,3} ~ e^\ss\right),
\end{align*}
where $c = 1 + \frac{3}{n} (e^\ss -1)$.
Setting $n \ge 100 (e^\ss - 1)/\ss$, we get that $c \le 1 + \ss/20$.
Since $x_{0,1} = y_{0,1}$, we get that
\begin{align*}
    \lone{x_1 - y_1}
        & = \frac{e^{\ss}}{c} \lone{x_0 - y_0} \\
        & \ge \frac{1 + \ss}{1 + \ss /20} \lone{x_0 - y_0} \\
        & \ge \lone{x_0 - y_0} + \frac{\ss}{4} \lone{x_0 - y_0}.
\end{align*}
Moreover, for KL-divergence we have
\begin{align*}
    \dkl(x_1,y_1)
        & = \frac{e^{\ss}}{c}  \dkl(x_0,y_0) \\
        & \ge (1 + {\ss}/{4}) \dkl(x_0,y_0).
\end{align*}
\end{proof}


Although~\Cref{lemma:MD-non-cont} says that mirror descent update is not contractive even for linear functions, it does not preclude the possibility that mirror descent is stable. Indeed, the following lemma shows that mirror descent enjoys similar stability guarantees to SGD for linear functions. Extending this stability result to general convex functions is an interesting open question.
\begin{lemma}
\label{lemma:MD-linear}
    Let $\Ds = (z_1,\dots,z_n)$ and 
    $\Ds' = (z'_1,\dots,z_n)$ be neighboring 
    datasets where $x_i \in \R^d$ and $\linf{z_i} \le \lip$. Consider the functions
    $f(x;z) = \< z, x\>$. Let $\{x_k\}_{k=0}^T$ be the iterates of~\Cref{alg:SMD} on $\Ds$ with $x_0 = \frac{1}{d} \cdot 1$ for $R$ rounds and $\ss >0$. Similarly, Let $\{y_k\}_{k=0}^T$ be the iterates of~\Cref{alg:SMD} on $\Ds'$ with $y_0 = \frac{1}{d} \cdot 1$ for $R$ rounds and $\ss >0$. Then after $R$ rounds ($T = Rn$ iterates),
    \begin{equation*}
       \lone{x_T - y_T}^2 
            \le \dkl(x_T,y_T) + \dkl(y_T,x_T) 
            \le 4 \ss^2 \lip^2 R^2.
    \end{equation*}
\end{lemma}

\begin{proof}
    First, note that 
    \begin{equation*}
      \log \frac{x_{k}}{y_{k}}
           = \ss \sum_{i=1}^{k-1} (g'_i - g_i) + C, 
    \end{equation*}
    where $C$ is a constant vector,
    $g_i$ and $g'_i$ are the (sub)-gradients for $\Ds$ and $\Ds'$, respectively.
    Thus we have that
    \begin{align*}
    \dkl(x_T,y_T) + \dkl(y_T,x_T) 
        & = \< x_{k} - y_{k}, \log \frac{x_{k}}{y_{k}} \> \\
        & = \ss \< x_{k} - y_{k}, \sum_{i=1}^{k-1} (g'_i - g_i) \> \\
        & \le \ss \sqrt{ \dkl(x_T,y_T) + \dkl(y_T,x_T)}
            \sum_{i=1}^{k-1} \linf{g'_i - g_i} \\
        & \le 2 \ss \lip R   \sqrt{\dkl(x_T,y_T) + \dkl(y_T,x_T)}~  ,
    \end{align*}
    where the first inequality follows from holder's inequality and the strong convexity of KL-divergence with respect to $\lone{\cdot}$ (this is Pinsker's inequality; 
    \iftoggle{arxiv}{see e.g.,~\cite{Duchi19}}{see e.g.,~\citealp{Duchi19}}) and the second inequality follows since the
    first sample $z_1$ (or $z'_1$) appears $R$ times.
    The claim follows.
\end{proof}

\begin{algorithm}
	\caption{Stochastic Mirror Descent}
	\label{alg:SMD}
	\begin{algorithmic}[1]
		\REQUIRE Dataset $\Ds=(\ds_1, \ldots, \ds_n)\in \domain^n$,
		step sizes $\ss$, 
		initial point $x_0$,
		number of rounds $R$;
		\STATE $k \gets 0$
		\FOR{$r=1$ to $R$\,}
		    \STATE Sample a random permutation $\pi: [n] \to [n]$
            \FOR{$i=1$ to $n$\,}
        	    \STATE Set $ g_k = \nabla f(x_k; z_{\pi(i)}) $
        	    \STATE Find 
        	    $x_{k+1} \defeq \argmin_{x \in \simplex_{d-1}} 
        	    \{\< g_k, x - x_k\> + \frac{1}{\ss} \db(x,x_k)  \}$ where $h(x) = \sum_{j=1}^d x_j \log x_j$
        	   	\STATE $k \gets k + 1$
            \ENDFOR
        \ENDFOR
            \RETURN $\bar{x}_T=\frac{1}{T}\sum_{k=1}^{T} x_k$
	\end{algorithmic}
\end{algorithm}

\section{Rates for General $\boldsymbol{\ell_p}$-Geometry}
\label{sec:general-geom}
In this section, we extend our algorithms to work for general $\ell_p$-geometries for $p >1$. Here, the optimization is over the domain $\xdomain = \{x \in \R^d: \norm{x}_p \le 1 \}$ and we consider functions $f : \xdomain \to \R$ that are $\lip$-Lipschitz with respect to $\norm{\cdot}_p$, that is, $\norm{g}_q \le \lip$ for all $x$ and sub-gradient $g \in \partial f(x)$
where $1/p+1/q=1$.

\subsection{Algorithms for ERM for $1 \le p \le 2$}
To extend~\Cref{alg:noisy-MD} to work for general geometries, we need to bound the sensitivity of the gradients. Consider $1 \le p \le 2$ then $q >2$ which implies that $\ltwo{g} \le d^{1/2 - 1/q} \norm{g}_q$, that is, the function is $d^{1/2 - 1/q} \lip$ with respect to $\ltwo{\cdot}$.

\begin{algorithm}
	\caption{Noisy Mirror Descent for General Geometries}
	\label{alg:noisy-MD-general-geom}
	\begin{algorithmic}[1]
		\REQUIRE Dataset $\Ds=(\ds_1, \ldots, \ds_n)\in \domain^n$,
		$1 < p$ and 
		convex set $\xdomain = \{ x \in \R^d: \norm{x}_p \le 1\}$,
		convex function $h: \xdomain \to \R$,
		step sizes $\{ \ss_k \}_{k=1}^T$, 
		batch size $b$, 
		initial point $x_0$,
		number of iterations $T$;
		\STATE Find $q \ge 1$ such that $1/q + 1/p = 1$
        \FOR{$k=1$ to $T$\,}
        	\STATE Sample $S_1,\dots,S_b\sim\mbox{Unif}(\Ds)$ 
        	\STATE Set $\hat g_k = \frac{1}{b} \sum_{i=1}^b \nabla f(x_k; S_i) + \noise_i$
        	where $\noise_i \sim \normal(0,\sigma^2 I_d)$
        	with $ \sigma =  {100 \lip\sqrt{d^{1-2/q} \log(1/\delta)}}/{ b \diffp}$
        	\STATE Find 
        	$x_{k+1} \defeq \argmin_{x \in \mc{X}} 
        	\{\<\hat g_k, x - x_k\> + \frac{1}{\ss_k} \db(x,x_k)  \}$
            \ENDFOR
            \RETURN $\bar{x}_T=\frac{1}{T}\sum_{k=1}^{T} x_k$ (convex)
            \RETURN $\hat{x}_T=\frac{2}{T(T+1)}\sum_{k=1}^{T} k x_k$ (strongly convex) 
	\end{algorithmic}
\end{algorithm}

\begin{theorem}
\label{thm:noisy-MD-general-geom}
    Let $1 < p \le 2$, $h: \xdomain \to \R$ be $1$-strongly convex with respect 
    to $\norm{\cdot}_p$, $x\opt = \argmin_{x \in \xdomain} \hat F(x;S)$, and assume $\db(x\opt,x_0) \le \rad^2$.
    Let $f(x;z)$ be convex and $\lip$-Lipschitz 
    with respect to $\norm{\cdot}_p$ for all $z \in \domain$.
    Setting $1 \le b$,
    $T = \frac{n^2}{b^2}$ and $\ss_k = \frac{D}{ \sqrt{T}} \frac{1}{\sqrt{\lip^2 + 4 d^{2/q} \sigma^2 \log d}}$,
    \cref{alg:noisy-MD-general-geom} is $(\diffp,\delta)$-DP 
    and 
    \begin{equation*}
        \E[\hat F(\bar x_T;S) - \hat F(x\opt;S)]
            \le \lip \rad  \cdot O \Bigg( \frac{b}{n}
                + \frac{\sqrt{d \log\tfrac{1}{\delta} (1 + \log d \cdot \indic{p < 2} )} }{n \diffp} \Bigg).
    \end{equation*}
    Moreover, if $f(x;z)$ is $\sc$-strongly convex relative to $h(x)$, then
    setting $\ss_k = \frac{2}{\sc (k+1)}$
    \begin{align*}
        \E[\hat F(\hat x_T;S) - \hat F(x\opt;S)]
            & \le O \left( \frac{\lip^2 b^2}{\sc n^2} + \frac{\lip^2 d \log \tfrac{1}{\delta} (1 + \log d \cdot \indic{p < 2} )  }{\sc n^2 \diffp^2} \right).
    \end{align*}
\end{theorem}

\begin{proof}
    Following the proof of~\Cref{thm:noisy-MD}, privacy follows from similar arguments, and for utility we need to upper bound $\E[\norm{\tilde g_k}_q]$. 
    Note that for $p=q=2$ we have $\E[\norm{\tilde g_k}_q^2 ] \le d$. Otherwise we have
    \begin{equation*}
    \E[\norm{\tilde g_k}_q^2 ] 
        \le 2 \lip^2 + 2 \E[\norm{\noise_k}_q^2 ] 
        \le 2 \lip^2 + 2 d^{2/q}  \E[\linf{\noise_k}^2] 
        \le 2 \lip^2 + 2 d^{2/q} \E[\linf{\noise_k}^2]
        \le 2 \lip^2 + 8 d^{2/q}\sigma^2 \log d.
    \end{equation*}
    Now we complete the proof for $p<2$. The same proof works for $p=2$.
   The previous bound implies
    \begin{align*}
     \E[\hat F(\bar x_T;S) - \hat F(x\opt;S)]
            & \le \frac{\rad^2}{T \ss} + \ss \lip^2  
                + 4 \ss d^{2/q} \sigma^2 \log d \\
            & \le  {2\rad \sqrt{(\lip^2 + 4 d^{2/q} \sigma^2 \log d) /T}} \\
            & \le \lip \rad \cdot O \left(\frac{b}{n}
                + \frac{\sqrt{d \log d \log \tfrac{1}{\delta}} }{n \diffp}\right) ,
    \end{align*}
    where the second inequality follows from the choice of $\ss$.
    For the second part, \cref{lemma:conv-MD-sc} implies that
    \begin{align*}
        \E[\hat F(\hat x_T;S) - \hat F(x\opt;S)]
            & \le \frac{\lip^2}{\sc} O \left( \frac{ b^2}{n^2} + \frac{d \log d \log \tfrac{1}{\delta}}{n^2 \diffp^2} \right)
            \!.
            \qedhere
    \end{align*}

\end{proof}

\subsection{Algorithms for SCO}
We extend~\Cref{alg:phased-erm-MD} to work for general $\ell_p$-geometries by using the general noisy mirror descent (\Cref{alg:noisy-MD-general-geom}) to solve the optimization problem at each phase. The following theorem 
proves our main result for $\ell_p$-geometry, that is, \Cref{thm:local-MD-general-geom}.
\begin{theorem}
\label{thm:local-MD-general-geom-rephrase}
    Let $1 < p \le 2$.
    Assume $\diam_p(\mc{X}) \le \rad$ and 
    $f(x;z)$ is convex and $\lip$-Lipschitz with respect to 
    $\norm{\cdot}_p$ for all $z \in \domain$.
    If we set 
    \begin{equation*}
        \ss = \frac{\rad}{\lip} \min 
        \left\{ {1/\sqrt{(p-1)n}}, {\diffp}/{\sqrt{d  \log \tfrac{1}{\delta} (1 + \log d \cdot \indic{p < 2} )}} \right\},
    \end{equation*}
    then~\cref{alg:phased-erm-MD-general-geom} requires $O(\log n \cdot \min(n^{3/2} \sqrt{\log d}, n^2 \diffp/\sqrt{d})) $ gradients and its output has
    \begin{equation*}
       \E[F(x_k) - F(x\opt)] 
       = \lip \rad  \cdot O \Bigg( \frac{1}{\sqrt{(p-1) n}}
            + \frac{ \sqrt{d \log \tfrac{1}{\delta} (1 + \log d \cdot \indic{p < 2} ) }}{(p-1) n \diffp} \Bigg). 
    \end{equation*}
\end{theorem}

\begin{proof}
    The proof follows from identical argument to the proof of~\Cref{thm:local-MD} using the fact that $h_i(x) = \frac{1}{2(p-1)} \norm{x - x_{i-1}}^2_p$ is $1$-strongly convex with respect to $\norm{\cdot}_p$.
\end{proof}

\begin{algorithm}
	\caption{Localized Noisy Mirror Descent}
	\label{alg:phased-erm-MD-general-geom}
	\begin{algorithmic}[1]
		\REQUIRE Dataset $\Ds=(\ds_1, \ldots, \ds_n)\in \domain^n$,
		$1 \le p$,
		constraint set $\xdomain$,
		step size $\ss$, initial point $x_0$;
		\STATE Set $k = \ceil{\log n}$
        \FOR{$i=1$ to $k$\,}
        	\STATE Set $\diffp_i = 2^{-i} \diffp$, 
        	        $n_i= 2^{-i} n$, 
        	        $\ss_i = 2^{-4i} \ss$ 
        	\STATE Apply~\cref{alg:noisy-MD-general-geom} with $(\diffp_i,\delta)$-DP, batch size $b_i =\max(\sqrt{n_i/\log d}, \sqrt{d/\diffp_i})$, $T = n_i^2/b_i^2$ and $h_i(x) = \frac{1}{2(p-1)} \norm{x - x_{i-1}}^2_p$
        	for solving the ERM over 
        	$ \mc{X}_i= \{x\in \xdomain: \norm{x - x_{i-1}}_p \le {2\lip \ss_i n_i(p - 1)} \}$:\\
        	$
        	    \displaystyle
        	    F_i(x) =  \frac{1}{n_i} \sum_{j=1}^{n_i} f(x;z_j) + \frac{1}{\ss_i n_i (p-1)} \norm{x - x_{i-1}}_p^2
        	$
        	\STATE Let $x_i$ be the output of the private
        	algorithm
            \ENDFOR
            \RETURN the final iterate $x_k$
	\end{algorithmic}
\end{algorithm}

\section{Proofs of~\cref{sec:non-smooth}}
\label{sec:proofs-non-smooth}

\iftoggle{arxiv}{}{
\subsection{Proof of~\cref{thm:noisy-MD}}
\label{proof:thm-noisy-MD}
\begin{proof}
    The privacy proof follows directly using moments accountant, that is, Theorem 1 in \cite{AbadiChGoMcMiTaZh16}, by noting the the $\ell_2$-norm of the gradients is bounded by $\ltwo{\nabla f(x; z_i)} \le \linf{\nabla f(x; z_i)} \sqrt{d}
    \le \lip \sqrt{d}$ for all $x \in \xdomain$ and $z \in \domain$.
   %
    Now we analyze the utility
    of the algorithm. To this end, we have that
    $\E[\linf{\hat g_k}^2] \le 2 \lip^2 + 2 \E[\linf{\noise_k}^2] \le 2 \lip^2 + 4 \sigma^2 \log d$. \cref{lemma:conv-MD} now implies that
    \begin{align*}
     \E[\hat F(\bar x_T;S) - \hat F(x\opt;S)]
            & \le \frac{\rad^2}{T \ss} + \ss \lip^2  
                + 2 \ss \sigma^2 \log d \\
            & \le 2 \frac{\rad \sqrt{\lip^2 + 2 \sigma^2 \log d}}{\sqrt{T}} \\
            & \le O \left( \lip \rad \left(\frac{b}{n}
                + \frac{\sqrt{d \log d \log(1/\delta)} }{n \diffp}\right) \right),
    \end{align*}
    where the second inequality follows from the choice of $\ss$.
    Now we prove the claim for strongly convex functions.
    \cref{lemma:conv-MD-sc} implies that
    \begin{align*}
  \E[\hat F(\hat x_T;S) - \hat F(x\opt;S)]
            & \le O \left( \frac{\lip^2 b^2}{\sc n^2} + \frac{\lip^2 d \log d \log(1/\delta)}{\sc n^2 \diffp^2} \right).
            \qedhere
    \end{align*}
\end{proof}
}

\subsection{Proof of~\cref{lemma:conv-MD-sc}}
\label{sec:proof-conv-MD-sc}
\begin{proof}
    First, by strong convexity we have
    \begin{align}
    \label{eq:conv-main}
        f(x_k) - f(x\opt) 
            & \le \<\nabla f(x_k), x_k - x\opt \> 
                - \sc \db(x\opt,x_k)  \nonumber \\
            & = \<g_k, x_k - x\opt \>
                + \<\nabla f(x_k) -  g_k, x_k - x\opt \> 
                - \sc \db(x\opt,x_k).
    \end{align}
    Let us now focus on the term $\< g_k, x_k - x\opt \>$.
    The definition of $x_{k+1}$ implies that for all $y \in \mc{X}$
    \begin{align*}
        \< g_k + \frac{1}{\ss_k}(\nabla h(x_{k+1}) - \nabla h(x_k)), y - x_{k+1} \> \ge 0.
    \end{align*}
    Substituting $y=x\opt$, we have
    \begin{align*}
    \< g_k , x_k - x\opt \>
        & =  \< g_k , x_k - x_{k+1} \> 
             +  \< g_k , x_{k+1} - x\opt \> \\
        & \le \< g_k , x_k - x_{k+1} \>
             + \frac{1}{\ss_k} \<\nabla h(x_{k+1}) - \nabla h(x_k), x\opt - x_{k+1} \> \\
        & \stackrel{(i)}{=} \< g_k , x_k - x_{k+1} \> 
            + \frac{1}{\ss_k} \left( \db(x\opt,x_k) - \db(x\opt,x_{k+1}) - \db(x_{k+1},x_k) \right) \\
        & \stackrel{(ii)}{\le} \frac{\ss_k}{2} \linf{ g_k}^2
            + \frac{1}{2 \ss_k} \lone{x_k - x_{k+1}}^2
            + \frac{1}{\ss_k} \left( \db(x\opt,x_k) - \db(x\opt,x_{k+1}) - \db(x_{k+1},x_k) \right) \\
        & \stackrel{(iii)}{\le} \frac{\ss_k}{2} \linf{g_k}^2
            + \frac{1}{\ss_k} \left( \db(x\opt,x_k) - \db(x\opt,x_{k+1})  \right),
    \end{align*}
    where $(i)$ follows from the definition of bregman 
    divergence, $(ii)$ follows from Fenchel-Young inequality,
    and $(iii)$ follows since $h(x)$ is $1$-strongly convex
    with respect to $\lone{\cdot}$.
    Substituting into~\eqref{eq:conv-main},
    \begin{align*}
     f(x_k) - f(x\opt)
            & \le \frac{\ss_k}{2} \linf{ g_k}^2
                + \<\nabla f(x_k) -  g_k, x_k - x\opt \> 
                + \frac{1}{\ss_k} \left( \db(x\opt,x_k) - \db(x\opt,x_{k+1})  \right)
                - \sc \db(x\opt,x_k).
    \end{align*}
    Multiplying by $k$ and summing from $k=1$ to $T$, we get
    \begin{align*}
    \sum_{k=1}^T k ( f(x_k) - f(x\opt) ) 
            & \le \frac{1}{2 \sc} \sum_{k=1}^T        \linf{ g_k}^2
                + \<\nabla f(x_k) -  g_k, x_k - x\opt \> \\
            & \quad  + \frac{\sc}{2} \left( k (k-1) \db(x\opt,x_k) - k (k+1) \db(x\opt,x_{k+1}) \right) \\
            &  \le \frac{1}{2 \sc} \sum_{k=1}^T        \linf{ g_k}^2
                + \<\nabla f(x_k) -  g_k, x_k - x\opt \>.
    \end{align*}
    The claim now follows by taking expectations and using Jensen's inequality.
\end{proof}

\iftoggle{arxiv}{}{
\subsection{Proof of~\cref{lemma:opt-dis-MD}}
\label{sec:proof-lemma-opt-dis-MD}
First, we prove that $\hat x_i \in \mc{X}_i$.
The definition of $\hat x_i$ implies that 
\begin{equation*}
    \frac{1}{n_i} \sum_{j=1}^{n_i} f(\hat x_i;z_j) + \frac{1}{\ss_i n_i (p-1)} \norm{\hat x_i - x_{i-1}}_p^2
    \le \frac{1}{n_i} \sum_{j=1}^{n_i} f(x_{i-1};z_j).
\end{equation*}
Since $f(x;z)$ is $\lip$-Lipschitz, we get
\begin{equation*}
    \frac{1}{\ss_i n_i (p-1)} \norm{\hat x_i - x_{i-1}}_p^2
    \le \lip \lone{\hat x_i - x_{i-1}} 
    \le 2 \lip \norm{\hat x_i - x_{i-1}}_p
\end{equation*}
where the last inequality follows from the choice of $p$ 
(since $\lone{z} \le d^{1 - 1/p} \norm{z}_p \le 2 \norm{z}_p$ for all
$z \in \R^d$),
hence we get $\norm{\hat x_i - x_{i-1}}_p \le {2\lip \ss_i n_i(p-1)}$.
Thus, we have that $\hat x_i \in \mc{X}_i
= \{x: \norm{x - x_{i-1}}_p \le {2\lip \ss_i n_i(p-1)}\}$.
    
Now, note that the function $F_i(x)$ 
is $\sc_i$-strongly convex relative to 
$h_i(x) = \frac{1}{p-1} \norm{x - x_{i-1}}_p^2$ where
$\sc_i = \frac{1}{\ss_i n_i}$.
Moreover, the function $r_i(x) = \frac{1}{\ss_i n_i (p-1)} \norm{x - x_{i-1}}_p^2$ is $4 \lip$-Lipschitz with 
respect to $\lone{\cdot}$ for 
$x \in \mc{X}_i$. 
Therefore using the bounds of~\cref{thm:noisy-MD}
for noisy mirror descent and observing that 
$F_i(x)$ is $\sc_i$-strongly convex 
with respect to $\norm{\cdot}_p$,
\begin{equation*}
    \frac{\sc_i}{2} \E [\norm{x_i - \hat x_i}_p^2]
        \le \E[F_i(x_i) - F_i(\hat x_i)]
        \le O \left( \frac{\lip^2 d \log d \log(1/\delta)}{n_i^2 \diffp_i^2 \sc_i} \right),
\end{equation*}
implying that $\E [\norm{x_i - \hat x_i}_p^2]
\le O \left( \frac{\lip^2 \ss_i^2 d \log d \log(1/\delta)}{\diffp_i^2} \right)$.

\subsection{Proof of~\cref{lemma:fun-subopt-MD}}
\label{sec:proof-lemma-fun-subopt-MD}
The proof follows from Theorems 6 and 7 in~\cite{ShwartzShSrSr09} by noting that the function $r(x;z_j) = f(x;z_j) + \frac{1}{\ss_i n_i (p-1)} \norm{x - x_{i-1}}_p^2$ is $\frac{1}{\ss_i n_i}$-strongly convex
and $O(\lip)$-Lipschitz with respect to $\lone{\cdot}$ over $\mc{X}_i$.
}

\iftoggle{arxiv}{}{

\section{Proofs of~\cref{sec:FW}}
\label{sec:proofs-FW}
\iftoggle{arxiv}{}{
To simplify notation, in this section we use the notion of \ed-indistinguishability; we say that 
two random variables $X$ and $Y$ are \ed-indistinguishable, denoted $X \approx_{(\eps,\delta)} Y$, if for every $\cO$, $\Pr(X \in \cO) \le e^{\eps} \Pr[Y \in \cO] +\delta$ and $\Pr(Y \in \cO) \le e^{\eps} \Pr[X \in \cO] +\delta$.}
\iftoggle{arxiv}{}{
\subsection{Proof of~\cref{lemma:FW-pure-priv}}
\label{sec:proof-lemma-FW-pure-priv}
The main idea for the privacy proof is that 
each sample in the set $S_{t,s}$ is used in the calculation of $v_{t,s}$ at most $N_{t,s} = 2^{t - |s|}$ times,
hence setting the noise large enough so that each iterate is $\frac{\diffp}{N_{t,s}}$-DP, we get that the final mechanism is $\diffp$-DP using basic composition. Let us now provide a more formal argument.
Let $\Ds=(\ds_1,\dots,\ds_{n-1},\ds_n),\Ds'=(\ds_1,\dots,\ds_{n-1},\ds'_n)$ be two neighboring datasets with iterates $x = (x_1,\dots,x_K)$
and $x' = (x'_1,\dots,x'_K)$, respectively. We prove that $x$ and $x'$
are $\diffp$-indistinguishable, i.e., $x \approx_{(\eps,0)} x'$.
Let $S_{t,s}$ be the set (vertex) that contains the last sample (i.e., $z_n$ or $z'_n$) and let 
$j = |s|$ denote the depth of this vertex. 
We will prove privacy given that the $n$'th sample is in $S_{t,s}$, which will imply our general privacy guarantee as this holds for every choice of $t$ and $s$.
    
Note that $|S_{t,s}| = 2^{-j} \bs$ and that this set is used in the calculation of $v_k$ for at most $2^{t-j}$ (consecutive) iterates, namely these are leafs that are descendants of the vertex $u_{t,s}$.
Let $k_0$ and $k_1$ be the first and last iterate such that the set $S_{t,s}$ is used for the calculation of $v_k$, hence $k_1 - k_0 + 1 \le  2^{t - j} $.
The iterates $(x_1,\dots,x_{k_0-1})$ and $(x'_1,\dots,x'_{k_0-1})$ do not depend on the last sample and therefore has the same distribution (hence $0$-indistinguishable).
Moreover, given that $(x_{k_0},\dots,x_{k_1}) \approx_{(\diffp,0)} (x'_{k_0},\dots,x'_{k_1})$, it is clear that
the remaining iterates $(x_{k_1+1},\dots,x_{K}) \approx_{(\diffp,0)} (x'_{k_1+1},\dots,x'_{K})$
by post-processing as they do depend on the last sample only through the previous iterates. It is therefore enough to prove that 
$(x_{k_0},\dots,x_{k_1}) \approx_{(\diffp,0)} (x'_{k_0},\dots,x'_{k_1})$. 
To this end, we prove that for each such iterate, $w_k \approx_{({\eps}/{ 2^{t - j}},0)} w'_k$, hence using post-processing and basic composition the iterates are 
$\eps$-indistinguishable as $k_1 - k_0 + 1 \le 2^{t - j}$.
Note that for every $k_0 \le k \le k_1$ the sensitivity $|\<c_i, v_k - v'_k \> | \le  \frac{\rad \lip}{2^{-j} \bs }$.
Hence, using privacy guarantees of report noisy max~[\citealp{DworkRo14}, claim 3.9], we have that $w_k \approx_{({\eps}/{ 2^{t - j}},0)} w'_k$ since 
$\lambda_{t,s} = \frac{ 2 \lip \rad 2^t }{ \bs \eps}$.
    
Now we prove the second part of the claim. 
Standard results for the expectation of the maximum of $m$ Laplace random variables imply that $\E[\<v_{t,s}, w_{t,s}\>] \le \min_{1 \le i \le m } \<v_{t,s},c_i\> + O(\frac{\lip \rad 2^t}{\bs \eps} {\log m})$.
Since $\xdomain = \mathsf{conv}\{c_1,\dots,c_m\}$, we know that for any $v \in \R^d$, $\argmin_{w \in \xdomain} \<w,v\>  \cap \{c_1,\dots,c_m\} \neq \emptyset $~[\citealp{TalwarThZh15}, fact 2.3] which proves the claim.
}

\iftoggle{arxiv}{}{
\subsection{Proof of~\cref{lemma:FW-var-tree}}
\label{sec:proof-lemma-FW-var-tree}

The claim follows directly from the following lemma.
\begin{lemma}
\label{lemma:coord-subg-tree}
    Let $(s,t)$ be a vertex and $\sigma^2 =  (\lip^2  + \sm^2 \rad^2)/\bs$.
    For every index $1 \le i \le d$, 
    \begin{equation*}
    \E \left[ e^{\lambda (v_{t,s,i} - \nabla F_i(x_{t,s}))} \right] 
        \le e^{O(1) \lambda^2 \sigma^2}.
    \end{equation*}
\end{lemma}
\Cref{lemma:coord-subg-tree} says that $v_{k,i} - \nabla F_i(x_{k})$ is $O(\sigma^2)$-sub-Gaussian for every $1 \le i \le d$, hence standard results imply that the maximum of $d$ sub-Gaussian random variables is 
$\E\linf{v_{t,s} - \nabla F(x_{t,s})} \le O(\sigma) \sqrt{\log d}$. The claim follows.

\begin{proof}[\Cref{lemma:coord-subg-tree}]
    Let us fix $i$ for simplicity and let 
    $B_{t,s} = v_{t,s,i} - \nabla F_i(x_{t,s})$.
    We prove the claim by induction on the depth of the vertex, i.e., $j = |s|$.
    If $j=0$ then $s = \emptyset$ which implies that 
    $v_{t,\emptyset} = \nabla f(x_{t,\emptyset}; S_{t,\emptyset})$ where $S_{t,\emptyset}$ is a sample of size $b$. Thus we have
    \begin{align*}
    \E[e^{\lambda B_{t,\emptyset}}]
       & = \E \left[ e^{\lambda (v_{t,\emptyset,i} - \nabla F_i(x_{t,\emptyset})} \right] \\
       & = \E \left[ e^{\lambda (\frac{1}{\bs} \sum_{s \in S_{t,\emptyset}} \nabla f_i(x_{t,\emptyset};s) - \nabla F_i(x_{t,\emptyset})} \right] \\
       & = \prod_{s \in S_{t,\emptyset}}
            \E[ e^{\frac{\lambda}{\bs} (\nabla f_i(x_{t,\emptyset};s) - \nabla F_i(x_{t,\emptyset}))}  ] \\
        & \le  e^{\lambda^2 \lip^2/2 \bs},
    \end{align*}
    where the last inequality follows since for a random variable $X \in [-\lip,\lip]$ and $\E[X]=0$, we have $\E[e^{\lambda X}] \le e^{\lambda^2 \lip^2/2}$~[\citealp{Duchi19}, example 3.6].
    Assume now we have $s$ with $|s| = j >0$ and let $s = s' c$ where $c \in \{0,1\}$. If $c=0$ the claim clearly holds so we assume $c=1$.
    Recall that in this case
    $v_{t,s} = v_{t,s'} + \nabla f(x_{t,s};S_{t,s}) - \nabla f(x_{t,s'};S_{t,s})$, hence
    $B_{t,s} = v_{t,s,i} - \nabla F_i(x_{t,s})
    = B_{t,s'} + \nabla f_i(x_{t,s};S_{t,s}) - \nabla f_i(x_{t,s'};S_{t,s}) - \nabla F_i(x_{t,s}) + \nabla F_i(x_{t,s'})$
    Letting $S_{< {t,s}} = \cup_{(t_1,s_1) < (t,s)} S_{t_1,s_1}$ be the set of all samples used up to vertex $t,s$, the law of iterated expectation implies
    \begin{align*}
    \E[e^{\lambda B_{t,s}}] 
        & = \E[e^{\lambda( B_{t,s'} + \nabla f_i(x_{t,s};S_{t,s}) - \nabla f_i(x_{t,s'};S_{t,s}) - \nabla F_i(x_{t,s}) + \nabla F_i(x_{t,s'}) )}] \\
        & = \E \left[ \E[e^{\lambda( B_{t,s'} + \nabla f_i(x_{t,s};S_{t,s}) - \nabla f_i(x_{t,s'};S_{t,s}) - \nabla F_i(x_{t,s}) + \nabla F_i(x_{t,s'}) )}] \mid S_{<(t,s)}\right] \\
        & = \E \left[ \E[e^{\lambda B_{t,s'} )} \mid S_{<(t,s)} ] \cdot
        \E[e^{\lambda( \nabla f_i(x_{t,s};S_{t,s}) - \nabla f_i(x_{t,s'};S_{t,s}) - \nabla F_i(x_{t,s}) + \nabla F_i(x_{t,s'}) )} \mid S_{<{t,s}}  ] \right] \\
        & =  \E[e^{\lambda B_{t,s'} )}] \cdot
        \E[e^{\lambda( \nabla f_i(x_{t,s};S_{t,s}) - \nabla f_i(x_{t,s'};S_{t,s}) - \nabla F_i(x_{t,s}) + \nabla F_i(x_{t,s'}) )} \mid S_{<{t,s}}].
    \end{align*}
    Since $f(\cdot;s)$ is $\sm$-smooth with respect to $\lone{\cdot}$, we have that $|\nabla f_i(x_{t,s};S_{t,s}) - \nabla f_i(x_{t,s'};S_{t,s})| \le \sm \lone{x_{t,s} - x_{t,s'}}$.
    Moreover, as $u_{t,s}$ is the right son of $u_{t,s'}$,
    the number of updates between $x_{t,s}$ and $x_{t,s'}$
    is at most the number of leafs visited between these two vertices
    which is $2^{t-j}$. Hence we get that
    \begin{equation*}
    \lone{x_{t,s} - x_{t,s'}} 
       \le \rad \ss_{t,s'} 2^{t-j}  
        \le \rad 2^{-j+2},    
    \end{equation*}
    which implies that  $|\nabla f_i(x_{t,s};S_{t,s}) - \nabla f_i(x_{t,s'};S_{t,s})| \le \sm \rad 2^{-j+2}$. 
    Since $\E[\nabla f_i(x_{t,s};S_{t,s}) - \nabla f_i(x_{{t,s'}};S_{t,s})  \mid S_{< {t,s}} ] = \nabla F_i(x_{t,s}) - \nabla F_i(x_{t,s'})$,
    by repeating similar arguments to the case $\ell=0$, we get that
    \begin{align*}
     \E[e^{\lambda( \nabla f_i(x_{t,s};S_{t,s}) - \nabla f_i(x_{t,s'};S_{t,s}) - \nabla F_i(x_{t,s}) + \nabla F_i(x_{t,s'}) )} \mid S_{< {t,s}}]
     & \le e^{O(1) \lambda^2 \sm^2 \rad^2 2^{-2j}/|S_{t,s}|} \\
     & \le e^{O(1) \lambda^2 \sm^2 \rad^2 2^{- j}/ \bs }.
    \end{align*}
    Overall we have that $\E[e^{\lambda B_{t,s}}] \le \E[e^{\lambda B_{t,s'}}] \cdot e^{O(1) \lambda^2 \sm^2 \rad^2 2^{-j}/\bs }$. Applying this inductively, we get that for every $(t,s)$
    \begin{equation*}
    \E[e^{\lambda B_{t,s}}]
        \le e^{O(1) \lambda^2 (\lip^2 + \sm^2 \rad^2)/ \bs}.
        \qedhere
    \end{equation*}
\end{proof}
}

\iftoggle{arxiv}{}{
\subsection{Proof of~\cref{lemma:FW-appr-priv}}
\label{sec:proof-lemma-FW-appr-priv}
For this proof, we use the following privacy amplification by shuffling.
\begin{lemma}[\citealp{FeldmanMcTa20}, Theorem 3.8] 
\label{lemma:amp-shuff}
    Let $\A_i: \range^{i-1} \times \domain \to \range$ for $i \in [n]$ be a sequence of algorithm such that $\A_i(w_{1:i-1},\cdot)$ is $(\diffp_0,\delta_0)$-DP for all values of $w_{1:i-1} \in \range^{i-1}$ with $\diffp_0 \le O(1)$. Let $\A_S : \domain^n \to \range^n$ be an algorithm that given $z_{1:n} \in \domain^n$, first samples a random permutation $\pi$, then sequentially computes $w_i = \A_i(w_{1:i-1},z_{\pi(i)})$ for $i \in [n]$ and outputs $w_{1:n}$. Then for any $\delta$ such that  $\diffp_0 \le \log(\frac{n}{16 \log(2/\delta)})$, the algorithm $\A_s$ is $(\diffp, \delta + 20 n \delta_0)$
    where $\diffp \le O ( {\diffp_0 \sqrt{\log(1/\delta)/n}})$.
\end{lemma}

We use the same notation as~\cref{lemma:FW-pure-priv} where 
$\Ds=(\ds_1,\dots,\ds_{n-1},\ds_n),\Ds'=(\ds_1,\dots,\ds_{n-1},\ds'_n)$ denote two neighboring datasets with iterates $x = (x_1,\dots,x_K)$
and $x' = (x'_1,\dots,x'_K)$. Here, we prove privacy after conditioning on the event that the $n$'th sample is sampled at phase $t$ and depth $j$.
We need to show that the iterates are $(\diffp,\delta)$-indistinguishable.
We only need to prove privacy for the iterates at phase $t$ as the iterates before phase $t$ do not depend on the $n$'th sample and the iterates after phase $t$ are $(\diffp,\delta)$-indistinguishable by post-processing. 

Let us now focus on the iterates at phase $t$. Let $u_1,\dots,u_p$ denote the vertices at level $j$ that has samples $S_1,\dots,S_p$ each of size
$|S_i| = 2^{-j} \bs$. We will have two steps in the proof. First, we use advanced composition to show that the iterates that are descendant of a vertex $u_i$ are $(\diffp_0,\delta_0)$-DP where roughly $\diffp_0 = {2^{j/2} \diffp}$. As we have $p = 2^j$ vertices at depth $j$, we then use the amplification by shuffling result to argue that the final privacy guarantee is 
$(\diffp,\delta)$-DP (see Fig.~\ref{fig:tree-shuffle} for a demonstration of the shuffling in our algorithm).

Let $\A_i$ be the algorithm that outputs the iterates corresponding to the leafs that are descendants of $u_i$; we denote this output by $O_i$. Note that the inputs of $\A_i$ are the samples at $u_i$, which we denote as $S_i$, and the previous outputs $O_1,\dots,O_{i-1}$. In this notation, we have that $O_i = \A_i(O_1,\dots,O_{i-1},S_i)$. 
We let $\A_i$, $S_i$ and $O_i$ denote the above quantities when the input dataset is $\Ds_i$ and similarly $\A'_i$, $S'_i$ and $O'_i$ for $\Ds'$.
To prove privacy, we need to show that $(O_1,\dots,O_p) \approx_{(\diffp,\delta)} (O'_1,\dots,O'_p)$, that is $(O_1,\dots,O_p)$ and $(O'_1,\dots,O'_p)$ are $(\diffp,\delta)$-indistinguishable

To this end, we first describe an equivalent sampling procedure for the sets $S_1,\dots,S_p$. Given $r$ samples, the algorithm basically constructs the sets $S_1,\dots,S_p$ by sampling uniformly at random $p$ sets of size $r/p$ without repetition. Instead, we consider the following sampling procedure. First, we randomly choose a set of size $p(r-1)$ samples that does not include the $n$'th sample and using this set we randomly choose $r/p - 1$ samples for each set $S_i$. Then, we shuffle the remaining $p$ samples and add each sample to the corresponding set. It is clear that this sampling procedure is equivalent.
We prove privacy conditional on the output of the first stage of the randomization procedure which will imply privacy unconditionally.

Assuming without loss of generality that the samples which remained in the second stage are $z_{n-p+1},\dots,z_n$, and letting $\pi: [p] \to \{n-p+1,\dots,n\}$ denote the random permutation of the second stage, the algorithms $\A_i$ and $\A'_i$ can be written as a function of the previous outputs and the sample $z_{\pi(i)}$. This is true since the $\Ds$ and $\Ds'$ differ in one sample and therefore the first $r/p-1$ samples in the sets $S_i$ and $S'_i$ are identical. Thus, we can write $O_i = \A_i(O_1,\dots,O_{i-1},z_{\pi(i)})$.

Using the above notation, we are now ready to prove privacy. First, we show privacy for each $i$ using advanced composition. Similarly to~\cref{lemma:FW-pure-priv}, as each iterate $k$ which is a leaf of $u_i$ has sensitivity  $|\<c_i, v_k - v'_k \> | \le  \frac{\rad \lip}{2^{-j} \bs }$, we have that $x_k$ and $x_k'$ are 
$\frac{\diffp}{2^{T/2-j} {\log(n/\delta)}}$-indistinguishable since 
$\lambda_{t,s} = \frac{ \lip \rad  2^{T/2} \log(n/\delta)}{ \bs \diffp}$. Since there are $2^{t-j}$ leafs of $u_i$, advanced composition (\cref{lemma:advanced-comp}) implies that $O_i \approx_{(\diffp_0, \delta_0)} O'_i$ 
where $\diffp_0 = \frac{\diffp}{2^{T/2-j} {\log(n/\delta)}} \sqrt{2^{t-j} \log(1/\delta_0)} 
\le \frac{O(\diffp)}{\sqrt{\log(1/\delta)} 2^{-j/2}} $ by setting $\delta_0 = \delta/n$.

Finally, we can use the amplification by shuffling result to finish the proof.
First, note that we need $\diffp_0 \le \log(\frac{2^j}{16 \log(2/\delta)})$
to be able to use~\cref{lemma:amp-shuff}.
If $2^j \le O(\log(1/\delta))$ then we do not need the amplification by shuffling result as $\diffp_0 \le O( \diffp 2^{j/2}/\sqrt{\log(1/\delta)}) \le O(\diffp)$.
Otherwise $2^j$ is large enough so that we can use~\cref{lemma:amp-shuff}.
Since each $\A_i$ and $\A'_i$ are $(\diffp_0,\delta_0)$-DP and since the second stage shuffles the inputs to each algorithm, \cref{lemma:amp-shuff} now implies that the outputs of the algorithms $\A_i$ and $\A'_i$ are
$(\diffp_f,\delta + 20 n \delta_0)$-DP where $\diffp_f \le \frac{\diffp_0 \sqrt{\log(1/\delta)}}{2^{j/2}} \le O(\diffp)$ which proves the claim.
}

\iftoggle{arxiv}{}{
\subsection{Proof of~\cref{thm:FW-appr}}
\label{sec:thm-FW-appr}
The assumptions on $\sm$ ensure that $2^T \le \bs$ and the assumptions on $\diffp$ ensure $\diffp \le 2^{-T/2} \log(n/\delta)$ hence the 
privacy follows from~\cref{lemma:FW-appr-priv}.
The utility analysis is similar to the proof of~\cref{thm:FW-pure}.
Repeating the same arguments in the proof of~\cref{thm:FW-pure} while using the new value of $\lambda_{t,s}$, we get
\begin{align*}
    \E[F(x_{K}) - F(x\opt)] 
        \le O \left( \rad (\lip + \sm \rad) \frac{\sqrt{\log d}}{\sqrt{\bs}} 
        + \frac{\sm \rad^2}{2^T} 
        + \rad \lip  \frac{2^{T/2} \log(n/\delta) \log m}{\bs \diffp} \right).
\end{align*}
As the number of samples is upper bounded by $T^2 \cdot \bs$, we set
 $T = \frac{2}{3} \log \left( \frac{\bs \diffp \sm \rad}{ \lip \log(n/\delta) \log m} \right)$
and $\bs = n / \log^2 n$ to get the first part of the theorem. Note that the condition on $\sm$ ensure the term inside the log is greater than $1$.
}

}

\section{Proofs for~\cref{sec:LB}}
\label{sec:proofs-LB}

\subsection{Proofs for~\cref{lemma:LB-med}}
\label{sec:proof-lemma-LB-med}
Without loss of generality, we assume that $\rad = 1$. Moreover, similarly to the proof of~\Cref{thm:LB-smooth},
we prove lower bounds on the sample complexity to achieve a certain error which will imply our lower bound on the utility.
For an algorithm $\A$ and data $\Ds \in \mc{Z}^n$, 
define the error of $\A$:
\newcommand{\err}{\mathsf{Err}}
\begin{equation*}
	\err(\A,\Ds) = 
	\E\left[ \sum_{j=1}^d |\bar{z}_j| 
		\indic {\sign(\A(\Ds)_j) \neq \sign(\bar{z}_j)} \right] .
\end{equation*} 
The error of a $\A$ for datasets of size $n$ is 
$\err(\A, n) = \sup_{\Ds \in \mc{Z}^n} \err(\A,\Ds)$.

We let $n\opt(\alpha,\diffp)$ denote the minimal $n$
such that there is an $(\diffp,\delta)$-DP (with $\delta = n^{-\omega(1)}$) Aansim $\A$
such that $\err(\A,n\opt(\alpha,\diffp)) \le \alpha$.
We prove the following lower bound on the sample
complexity which implies~\cref{lemma:LB-med}.
\begin{proposition}
\label{prop:sample-complexity-LB}
    Let $z_i \in \{-1/d, 1/d\}^d$, $\alpha \le 1$, and $\diffp \le 1$. Then
	\begin{equation*}
	n\opt(\alpha,\diffp) 
		\ge \Omega(1) \cdot \frac{\sqrt{d}} {\alpha \diffp \log d} .
	\end{equation*} 
\end{proposition}

The proof follows directly from the following two lemmas.
\begin{lemma}[\citet{TalwarThZh15}, Theorem 3.2]
\label{lemma:sample-complexity-LB-low-accuracy-sign}
	Let the assumptions of~\Cref{prop:sample-complexity-LB} hold. Then
	\begin{equation*}
	n\opt(\alpha = 1/4,\diffp = 0.1) 
		\ge \Omega(1) \cdot \frac{\sqrt{d}}{ \log d} .
	\end{equation*} 
\end{lemma}

The following lemma shows how to extend the above lower bound to arbitrary accuracy and privacy parameters.
\begin{lemma}
\label{lemma:low-to-high-accuracy-sign}
    Let $\diffp_0 \le 0.1$.
	For $\alpha \le \alpha_0/2$ and $\diffp \le \diffp_0/2$,
	\begin{equation*}
	n\opt(\alpha,\diffp ) 
	\ge \frac{\alpha_0 \diffp_0}{\alpha \diffp}
		n\opt(\alpha_0, \diffp_0)  .
	\end{equation*} 
\end{lemma}
\begin{proof}
    The proof follows the same arguments as in the proof of~\Cref{lemma:low-to-high-accuracy}. 
\end{proof}

\subsection{Proof of~\Cref{thm:LB-smooth}}
\label{sec:LB-smooth-app}
In this section, we prove~\cref{thm:LB-smooth}. We begin by recalling the lower bound of~\citet{TalwarThZh15} and showing how it implies~\cref{lemma:LB-smooth-high-dim}.

\citet{TalwarThZh15} consider the family of quadratic functions where $f(x;a_i,b_i) = (a_i^T x - b_i)^2$ where $a_i \in \R^d$ and $b_i \in \R$. We assume $\xdomain = \{x: \lone{x} \le \rad \}$, $\linf{a_i} \le C$, and $|b_i| \le C \rad$. Note that the function $f$ is $\lip$-Lipschitz and $\sm$-smooth with $\lip \le O(C^2 \rad)$ and $\sm \le O(C^2)$ and there is a choice of $a_i,b_i$ that attains these. Theorem 3.1 in~\cite{TalwarThZh15} gives a lower bound of $1/n^{2/3}$ when $C=1$, $\rad=1$, and $d \ge \wt \Omega(n^{2/3})$. 
For general values of $C$ and $\rad$, noticing that the function value is multiplied by $C^2 \rad^2$, the following lower bound follows as $\lip \rad = C^2 \rad^2$.
\begin{lemma}
\label{lemma:LB-smooth-high-dim}
    Let $\xdomain = \{x \in \R^d : \lone{x} \le \rad \}$ and $d \ge \wt \Omega(n^{2/3})$. There is family of convex functions $f : \xdomain \times \domain \to \R$ that is $\lip$-Lipschitz and $\sm$-smooth  with $\sm \le \lip/\rad$ such that any $(0.1,\delta)$-DP algorithm $\A$ with $\delta = o(1/n^2)$ has
    \begin{equation*}
        \sup_{\Ds \in \domain^n} \E \left[ \hat F(\A(\Ds);\Ds) - \min_{x \in \xdomain} \hat F(x;\Ds) \right] \ge \wt \Omega \left( \frac{\lip \rad}{n^{2/3}} \right).
    \end{equation*}
\end{lemma}

Now we proceed to prove~\cref{thm:LB-smooth} and we assume without loss of generality that $\lip=1$ and $\rad=1$. We use techniques from~\cite{SteinkeUl17} to extend the lower bound of~\cref{lemma:LB-smooth-high-dim} to hold for arbitrary $d$ and $\diffp$. To this end, instead of lower bounding the excess loss, it will be convenient to prove lower bounds on the sample size to achieve a certain excess loss $\alpha$. More precisely, given a dataset $\Ds \in \domain^n$ and algorithm $\A$, we define its empirical excess loss on $\Ds$
\newcommand{\excess}{\mathcal{E}}
\begin{equation*}
    \excess(\A,\Ds) = \E \left[ \hat F(\A(\Ds);\Ds) - \min_{x \in \xdomain} \hat F(x;\Ds) \right]. 
\end{equation*}
We also define its worst-case excess loss over all datasets of size $n$
\begin{equation*}
    \excess(\A,n) = \sup_{\Ds \in \domain^n}  \excess(\A,\Ds). 
\end{equation*}
We let $n\opt(\alpha,\diffp)$ be the minimal sample size that is required to achieve excess loss $\excess(\A,n\opt(\alpha,\diffp)) \le \alpha$ using an $(\diffp,\delta)$-DP algorithm $\A$ with $\delta = n^{-\omega(1)}$. We prove the following lemma which implies~\cref{thm:LB-smooth}.
\begin{lemma}
\label{lemma:LB-sample-size}
    Let the assumptions of~\cref{thm:LB-smooth} hold. Then 
    \begin{equation*}
    n\opt(\alpha,\diffp)
        \ge \begin{cases}
            \wt \Omega \left( \frac{1}{\alpha^{3/2} \diffp} \right) & \text{if } \alpha = 1/d \\
            \wt \Omega \left( \frac{\sqrt{d}}{\alpha \diffp} \right) & \text{if } \alpha \le 1/d \\    
            \end{cases}
    \end{equation*}
\end{lemma}

The proof of~\cref{lemma:LB-sample-size} basically follows from the following two Lemmas.
\begin{lemma}
\label{lemma:low-to-high-accuracy}
	For $0 < \alpha \le \alpha_0$ and $0 < \diffp \le \diffp_0 \le 0.1$,
	\begin{equation*}
	n\opt(\alpha,\diffp ) 
	\ge \Omega \left( \frac{\alpha_0 \diffp_0}{\alpha \diffp}
		n\opt(\alpha_0, \diffp_0) \right)  .
	\end{equation*} 
\end{lemma}

\begin{lemma}
\label{lemma:sample-size}
	We have that
	\begin{equation*}
	n\opt(\alpha = 1/d,\diffp=0.1 ) 
	\ge \wt \Omega \left( d^{3/2} \right).
	\end{equation*} 
\end{lemma}

\noindent
Before proving Lemmas~\ref{lemma:low-to-high-accuracy} and~\ref{lemma:sample-size}, let us finish the proof of~\cref{lemma:LB-sample-size}. First, consider the case $\alpha= 1/d$. \cref{lemma:sample-size} implies that 
\begin{equation*}
    n\opt(\alpha = 1/d,\diffp) 
        \ge  \Omega\left( \frac{n\opt(\alpha = 1/d,\diffp=0.1)}{\diffp} \right) 
        \ge \wt \Omega\left( d^{3/2}/\diffp  \right)
        = \wt  \Omega\left(\frac{1}{\alpha^{3/2} \diffp} \right).
\end{equation*} 
If $\alpha \le 1/d$, then similarly we have
\begin{equation*}
    n\opt(\alpha,\diffp)
        \ge \Omega\left( \frac{1}{d \alpha \diffp} \right) n\opt(\alpha = 1/d,\diffp=0.1)
        \ge \wt \Omega\left( \frac{\sqrt{d}}{\alpha \diffp} \right).
\end{equation*}
Hence~\cref{lemma:LB-sample-size} follows. Finally, we provide proofs for the remaining lemmas.

\begin{proof}[\cref{lemma:sample-size}]
    This lemma follows directly from~\cref{lemma:LB-smooth-high-dim}. Indeed, \cref{lemma:LB-smooth-high-dim} implies that if $d \ge \wt \Omega (n^{2/3})$ and $\diffp = 0.1$, the excess loss is lower bounded by $\excess(\A,n) \ge \wt \Omega(1/n^{2/3})$. Stated differently, if $n \le \wt O(d^{3/2})$ then $\excess(\A,n) \ge \wt \Omega(1/n^{2/3}) \ge \wt \Omega(1/d)$ which proves the claim.
\end{proof}

\begin{proof}[\cref{lemma:low-to-high-accuracy}]
    Given an $(\diffp,\delta)$-DP algorithm $\A$ with $\excess(\A,n) \le \alpha$, we show how to construct $\A'$ that is $(\diffp_0,4 \delta \diffp_0/\diffp)$-DP algorithm that works on datasets of size $n' = \Theta(\frac{\alpha \diffp}{\alpha_0 \diffp_0} n)$ such that $\excess(\A',n') \le \alpha_0$. This will prove the claim as we know that $n' \ge n(\alpha_0,\diffp_0)$. We now describe the construction of $\A'$. Given $\Ds' \in \domain^{n'}$ and $k>0$ to be chosen presently, we define a new dataset $\Ds$ as follows: the first $k n'$ samples are $k$ copies of $\Ds'$ and the remaining $n - kn'$ are new samples $z \in \domain$ that have the loss function $f(x;z) = 0$ for all $x \in \xdomain$. Clearly, these functions are convex, $0$-Lipschitz, and $0$-smooth. We then define $\A'(\Ds') = \A(\Ds)$. Note that for all $x$ we have that $\hat F(x;\Ds) = \frac{k n'}{n} \hat F(x;\Ds')$, which implies that
    \begin{align*}
    \excess(\A',\Ds')
        & = \E[ \hat F(\A(\Ds);\Ds') - \min_{x \in \xdomain} \hat F(x;\Ds')  ] \\
        & = \frac{n}{k n'} \E[ \hat F(\A(\Ds);\Ds) - \min_{x \in \xdomain} \hat F(x;\Ds)  ] \\
        & = \frac{n}{k n'} \excess(\A,\Ds) 
        \le \frac{n \alpha}{k n'}. 
    \end{align*}
    Therefore if $n' \ge n \alpha / k \alpha_0$ we get $\excess(\A',\Ds') \le \alpha_0$. Hence it remains to argue for privacy. Using the group privacy property of private algorithms~\citep{SteinkeUl17}(Fact 2.2), the algorithm $\A'$ is $(k\diffp, \frac{e^{k \diffp} - 1}{e^\diffp- 1} \delta)$-DP. Setting $k = \floor{\log (1+\diffp_0)/\diffp}$ implies the claim as $e^{k \diffp} - 1 \le \diffp_0$ and $k \diffp \le \diffp_0$.
\end{proof}

\end{document}